\documentclass[aap]{stylefile}
\RequirePackage{amsthm,amsmath,amsfonts,amssymb}
\RequirePackage{natbib}
\RequirePackage{mathtools}
\RequirePackage[shortlabels]{enumitem}
\RequirePackage{algorithm}
\RequirePackage{algorithmic}
\RequirePackage{bm}
\RequirePackage{graphicx}
\RequirePackage{tcolorbox}
\RequirePackage[colorlinks,citecolor=blue,urlcolor=blue]{hyperref}

\newcommand\blfootnote[1]{
  \begingroup
  \renewcommand\thefootnote{}\footnote{#1}
  \addtocounter{footnote}{-1}
  \endgroup
}

\usepackage{ragged2e}
\usepackage{bm}
\usepackage{multicol}
\usepackage{multirow}
\usepackage{pifont}
\newcommand{\xmark}{\ding{55}}
\newcommand{\cmark}{\ding{51}}

\DeclareMathOperator*{\argmin}{arg\,min}

\startlocaldefs
\theoremstyle{plain}
\newtheorem{theorem}{Theorem}[section]
\newtheorem{lemma}{Lemma}[section]
\newtheorem{corollary}{Corollary}[theorem]
\newtheorem{proposition}{Proposition}[section]

\theoremstyle{definition}

\newtheorem{assumption}{Assumption}[section]

\theoremstyle{remark}
\newtheorem*{remark}{Remark}

\newtheorem{fact}{Fact}[section]
\endlocaldefs

\begin{document}

\begin{frontmatter}
\title{Finite-Sample Analysis of Off-Policy Natural Actor-Critic with Linear Function Approximation
}

\begin{aug}
\author[A]{\fnms{Zaiwei} \snm{Chen}\ead[label=e1,mark]{zchen458@gatech.edu}}
,\author[A]{\fnms{Sajad} \snm{Khodadadian}\ead[label=e2,mark]{skhodadadian3@gatech.edu}} 
\and
\author[A]{\fnms{Siva Theja} \snm{Maguluri}\ead[label=e3,mark]{siva.theja@gatech.edu}}

\address[A]{
Geogia Institute of Technology,
\printead{e1,e2,e3}}

\blfootnote{Equal contribution between Zaiwei Chen and Sajad Khodadadian}

\end{aug}

\begin{abstract}
In this paper, we develop a novel variant of off-policy natural actor-critic algorithm with linear function approximation and we establish a sample complexity of $\mathcal{O}(\epsilon^{-3})$, outperforming all the previously known convergence bounds of such algorithms. In order to overcome the divergence due to deadly triad in off-policy policy evaluation under function approximation, we develop a critic that employs $n$-step TD-learning algorithm with a properly chosen $n$. We present finite-sample convergence bounds on this critic under both constant and diminishing step sizes, which are of independent interest. Furthermore, we develop a variant of natural policy gradient under function approximation, with an improved convergence rate of $\mathcal{O}(1/T)$ after $T$ iterations. Combining the finite sample error bounds of actor and the critic, we obtain the $\mathcal{O}(\epsilon^{-3})$ sample complexity. We derive our sample complexity bounds solely based on the assumption that the behavior policy sufficiently explores all the states and actions, which is a much lighter assumption compared to the related literature.
\end{abstract}

\end{frontmatter}

\section{Introduction}\label{sec:intro}
Reinforcement learning (RL) is a paradigm in which an agent aims at maximizing long term rewards via interacting with the environment. For solving the RL problem, there are value space methods such as $Q$-learning, and policy space methods such as actor-critic (AC) and its variants (e.g. natural actor critic (NAC)). In the AC framework, the actor aims at performing the policy update while the critic aims at estimating the value function of the current policy at hand. For AC type algorithms to perform well, the policy used to collect samples (called the behavior policy) must sufficiently explore the state-action space \citep{sutton2018reinforcement}. If the behavior policy coincides with the current policy iterate of AC, it is called on-policy sampling, otherwise it is called off-policy sampling.

In on-policy AC, the agent is restricted to use the current policy iterate to collect samples, which may not be exploratory. Moreover, on-policy sampling might be of high risk (e.g. self driving cars \citep{yurtsever2020survey}), high cost (e.g. robotics \citep{gu2017deep,levine2020offline}), or might be unethical (e.g. in clinical trials \citep{gottesman2019guidelines, liu2018representation, gottesman2020interpretable}). Off-policy AC, on the other hand, is more practical than on-policy sampling \citep{levine2020offline}. Specifically, off-policy sampling enables the agent to learn using the historical data, hence decouples the sampling process and the learning process. This allows the agent to learn in an off-line manner, and makes RL applicable in high-stake problems mentioned earlier. In addition, it is empirically observed that by using a suitable behavior policy, one can rectify the exploration issue in on-policy AC. As a result, off-policy learning successfully solved many practical problems in different areas, such as board game \citep{silver2017mastering}, city navigation \citep{mirowski2018learning}, education \citep{mandel2014offline}, and healthcare \citep{dann2019policy}.

In practice, AC algorithms are usually used along with function approximation to overcome the curse of dimensionality in RL \citep{bellman1957dynamic}. However, it has been observed that the combination of function approximation, off-policy sampling, and bootstrapping (also known as the deadly triad \citep{sutton2018reinforcement}) can result in instability or even divergence \citep{sutton2018reinforcement,baird1995residual}. In this work, we develop a variant of off-policy NAC with function approximation, and we establish its finite-sample convergence guarantee in the presence of the deadly triad.

\subsection{Main Contributions}
The main contributions of this paper are fourfold.

\textbf{Finite-Sample Bounds of Off-Policy NAC.} We develop a variant of NAC with off-policy sampling, where both the actor and the critic use linear function approximation, and the critic uses off-policy sampling. We establish finite-sample mean square bound of our proposed algorithm. Our result implies an $\tilde{\mathcal{O}}(\epsilon^{-3})$ sample complexity, which is the best known convergence bound in the literature for AC algorithms with function approximation.

\textbf{Novelty in the Critic.}
Off-policy TD with function approximation is famously  known to diverge due to deadly triad \citep{sutton2018reinforcement}. To overcome this difficulty, we employ $n$-step TD-learning, and show that a proper choice of $n$ naturally achieves convergence, and we present finite-sample bounds under both constant and diminishing stepsizes. To the best of our knowledge, we are the first to design a single time-scale off-policy TD with function approximation with provable finite-sample bounds.

\textbf{Novelty in the Actor.}
NAC under function approximation was developed in \cite{agarwal2021theory} by projecting the $Q$-values (gradients) to the lower dimensional space, and this involves the use of the discounted state visitation distribution, which is hard to estimate. We develop a new NAC algorithm for the function approximation setting that is instead based on the solution of  a projected Bellman equation \citep{tsitsiklis1997analysis}, which our critic is designed to solve.

\textbf{Exploration through Off-Policy Sampling.} We establish the convergence bounds under the minimum set of assumptions, viz., ergodicity under the behavior policy, which ensures sufficient exploration, and thus resolving challenges faced in on-policy sampling. As a result, learning can be done using a single trajectory of samples generated by the behavior policy, and we do not require constant reset of the system that was introduced in on-policy AC algorithms \citep{agarwal2021theory, wang2019neural} to ensure exploration. A similar observation about employing off-policy sampling to ensure exploration has been made in the tabular setting in \cite{khodadadian2021finite2}.

\subsection{Related Literature}\label{subsec:literature}

The two main approaches for learning an optimal policy in an RL problem are value space methods, such as $Q$-learning, and policy space methods, such as AC. The $Q$-learning algorithm proposed in \cite{watkins1992q} is perhaps the most well-known value space method. The asymptotic convergence of $Q$-learning was established in \cite{tsitsiklis1994asynchronous,jaakkola1994convergence,borkar2000ode,melo2008analysis}. As for finite-sample bounds, see \cite{wainwright2019stochastic,qu2020finite,chen2021finite,li2021tightening,chen2019finitesample} and the references therein. We next focus on related literature on AC-type of algorithms.

AC algorithms comprise two stages: actor and critic. The actor is responsible for the policy improvement, which is usually performed with the policy gradient (PG). The critic estimates the value function of the current policy (which provides the gradient), and uses TD-learning methods.  

\textbf{PG Methods.} The first PG algorithm with function approximation was proposed in \cite{sutton1999policy}, where the asymptotic convergence was established. A refined asymptotic analysis of PG methods has been further proposed in \cite{baxter2001infinite, pirotta2015policy, haarnoja2017reinforcement}. Natural policy gradient (NPG), which is a PG method with preconditioning, was proposed in \cite{kakade2001natural}. Recently, there has been a line of work to establish finite-sample convergence bound of NPG. In particular, sublinear convergence of NPG was established in \cite{azar2012dynamic, geist2019theory, agarwal2021theory, shani2020adaptive, zhang2020variational}, and geometric convergence of NPG was established in \cite{mei2020global, cen2021fast, bhandari2020note, lan2021policy, khodadadian2021on}.

\textbf{TD-Learning.} The policy evaluation problem within the critic is usually solved with TD-learning. In the on-policy setting, the asymptotic convergence of TD-learning was established in \cite{tsitsiklis1997analysis, tadic2001convergence, borkar2009stochastic}, and the finite-sample bounds were studied in \cite{dalal2018finite, lakshminarayanan2018linear, bhandari2018finite, srikant2019finite, hu2019characterizing, chen2021finite}. When TD-learning is used with off-policy sampling and function approximation, all the three elements of the deadly triad are present \citep{sutton2018reinforcement}. As a result, the algorithm can diverge. In order to overcome the divergence issue, numerous variants of TD-learning algorithms, such as  GTD \citep{sutton2008convergent}, TDC \citep{sutton2009fast}, and emphatic TD-learning \citep{sutton2016emphatic}, are proposed in the literature. However, all these algorithms require to maintain two iterates and hence are two time-scale algorithms, while our proposed algorithm is a single time-scale algorithm.

\begin{table}[t]
\centering
\caption{Sample complexity bounds of the AC-type algorithms using function approximation}\label{table: results2}
\renewcommand{\arraystretch}{1.4} 
\begin{tabular}{ |c|c|c|c|c| }
\hline
\multirow{2}{4 em}{\centering Algorithm} & \multirow{2}{4.5 em}{\centering Sampling Procedure}& \multirow{2}{6 em}{\centering References} & \multirow{2}{6 em}{\centering Sample Complexity $^{1,2}$} &  \multirow{2}{4em}{\centering Single Trajectory}  \\
   &  &  &  & \\
\hline
\multirow{4}{4 em}{\centering Actor Critic} & \multirow{3}{4.5 em}{\centering On-Policy} & \cite{konda2000actor} & Asymptotic & \cmark\\
\cline{3-5}
& &\cite{wang2019neural} & $\tilde{\mathcal{O}}(\epsilon^{-6})$& \xmark\\
\cline{3-5}
& & \cite{qiu2019finite,kumar2019sample} &$\tilde{\mathcal{O}}(\epsilon^{-4})$ & \xmark \\
\cline{2-5}
& Off-Policy&  \cite{maei2018convergent,zhang2020provably} & Asymptotic & \cmark\\
\hline
\multirow{5}{4 em}{\centering Natural Actor Critic} & \multirow{3}{4.5 em}{\centering On-Policy} & \cite{bhatnagar2009natural} & Asymptotic & \cmark\\
\cline{3-5}
& & \cite{wang2019neural}& $\tilde{\mathcal{O}}(\epsilon^{-14})$ & \xmark \\
\cline{3-5}
& & \cite{agarwal2021theory} & $\tilde{\mathcal{O}}(\epsilon^{-6})$ & \xmark \\
\cline{2-5}
& \multirow{2}{4.5 em}{\centering Off-Policy} & \cite{xu2021doubly} & $\tilde{\mathcal{O}}(\epsilon^{-4})$ & \xmark  \\
\cline{3-5}
& &  This work & $\tilde{\mathcal{O}}(\epsilon^{-3})$ & \cmark \\
\hline
\end{tabular}

\justify
$^1$ In this table, for the AC (respectively NAC) algorithms, sample complexity is the number of samples needed to find a policy $\pi$ such that $\mathbb{E}[\|\nabla V^{\pi}(\mu)\|^2] \leq \epsilon + \mathcal{E}_{\text{bias}}$ (respectively $\mathbb{E}[V^*(\mu)-V^{\pi}(\mu)]\leq \epsilon +\mathcal{E}_{bias}$), where $\mathcal{E}_{bias}$ is a non-vanishing error due to the function approximation. In the presence of a bias, one should be careful about interpreting the sample complexity. For a detailed illustration, see Appendix \ref{ap:sample_complexity} of this work and also Appendix C of \cite{khodadadian2021finite2}.\\
$^2$ Here $\Tilde{O}(\cdot)$ ignores all the logarithmic terms.

\end{table}

\textbf{On-Policy AC.} Several variants of AC were proposed in \cite{barto1983neuronlike, borkar1997actor, NIPS2009_3767, peters2008natural, NIPS2013_5184}. In the tabular setting, \cite{williams1990mathematical, borkar2009stochastic, borkar1997actor} studied the asymptotic convergence of AC algorithm. Furthermore, \cite{konda2000actor, bhatnagar2009natural} characterize the asymptotic convergence of on-policy AC under function approximation. Recently, there has been a flurry of work studying the finite-sample convergence of AC and NAC \citep{even2009online}. \cite{shani2020adaptive, lan2021policy, khodadadian2021finite} perform the finite sample analysis of NAC under tabular setting, and \cite{zhang2019convergence, qiu2019finite, kumar2019sample, liu2019neural, wang2019neural,xu2020non, xu2020improving, liu2020improved, wu2020finite} establish the finite-sample bounds of AC in function approximation setting. To the best of our knowledge, the best sample complexity bound of AC algorithms is provided in \cite{lan2021policy}, where the authors characterize an $\tilde{\mathcal{O}}(\epsilon^{-2})$ sample complexity. However, \cite{lan2021policy} only studies tabular RL in the on-policy setting.

\textbf{Off-policy AC.} Off-policy AC, was first proposed in \cite{degris2012off}. After that, there has been numerous extensions to that work such as DPG \citep{silver2014deterministic}, DDPG \citep{lillicrap2016continuous}, ACER \citep{wang2016sample}, TD3 \citep{fujimoto2018addressing}, IMPALA \citep{espeholt2018impala}, ACE \citep{imani2018off}, etc. The asymptotic convergence of off-policy AC was established for Gradient-AC in \cite{maei2018convergent}, and for AC with emphasis in \cite{zhang2020provably}. The first finite-sample bound of off-policy NAC was established in \cite{khodadadian2021finite2}. However, in \cite{khodadadian2021finite2} only tabular setting was studied. In the function approximation setting, \cite{xu2021doubly} provided the finite sample analysis of a doubly robust off-policy AC. \cite{liu2020off} also provided a convergence bound for off-policy AC, however their convergence bound does not involve a bound for the critic. A detailed comparison between our results and the related literature on off-policy AC-type algorithms with function approximation is presented in Table \ref{table: results2}.

\section{Main Results}\label{sec:NAC}

In this section, we present our main results. Specifically, in Section \ref{subsec:background} we briefly cover the background of RL and AC. In Section \ref{subsec:critic}, we present our algorithm design for the critic, which uses off-policy sampling with linear function approximation. In section \ref{subsec:actor}, we combine the critic with our actor update to form a variant of off-policy NAC with linear function approximation, and we present our finite-sample guarantees and sample complexity bounds. 

\subsection{Preliminaries}\label{subsec:background}
Consider modelling the RL problem as an infinite horizon MDP, which consists of a finite set of states $\mathcal{S}$, a finite set of actions $\mathcal{A}$, a set of unknown transition probability matrices $\mathcal{P}=\{P_a\in\mathbb{R}^{|\mathcal{S}|\times|\mathcal{S}|}\mid a\in\mathcal{A}\}$, an unknown reward function $\mathcal{R}:\mathcal{S}\times\mathcal{A}\mapsto \mathbb{R}$, and a discount factor $\gamma\in (0,1)$. Without loss of generality we assume that $\max_{s,a}|\mathcal{R}(s,a)|\leq 1$. For a given policy $\pi$, its state value function is defined by $V^\pi(s) =\mathbb{E}_\pi[\sum_{k=0}^\infty \gamma^k\mathcal{R}(S_k,A_k)\mid S_0=s]$ for all $s\in\mathcal{S}$, and its state-action value function is defined by $Q^\pi(s,a)=\mathbb{E}_\pi[\sum_{k=0}^\infty\gamma^k\mathcal{R}(S_k,A_k)\mid S_0=s,A_0=a]$ for all $(s,a)\in \mathcal{S}\times\mathcal{A}$. The goal of RL is to find an optimal policy $\pi^*$ which maximizes $V^\pi(\mu)=\sum_s \mu(s)V^{\pi}(s)$, where $\mu$ is an arbitrary fixed initial distribution over the state space. It was shown in the literature that the optimal policy is in fact independent of the initial distribution. See \cite{bertsekas1996neuro,puterman1995markov,sutton2018reinforcement} for more details for the MDP model of the RL problem.

To solve the RL problem, a popular approach is to use the AC framework \citep{konda2000actor}. In AC algorithm, we iteratively perform the policy evaluation and the policy improvement until an optimal policy is obtained. Specifically, in each iteration, we first estimate the $Q$-function (or the advantage function) of the current policy at hand, which is related to the policy gradient. Then we update the policy using gradient ascent over the space of the policies. NAC is a variant of AC where the gradient ascent step is performed with a properly chosen pre-conditioner. See \cite{agarwal2021theory} for more details about AC and NAC.

In AC framework, since we need to work with the $Q$-function and the policy, which are $|\mathcal{S}||\mathcal{A}|$ dimensional objects, the algorithm becomes intractable when the size of the state-action space is large \citep{bellman1957dynamic}. To overcome this difficulty, in this work we consider using linear function approximation for both the policy and the $Q$-function. Specifically, let $\{\phi_i\}_{1\leq i\leq d}$ be a set of basis functions, where $\phi_i\in\mathbb{R}^{|\mathcal{S}||\mathcal{A}|}$ for all $i$. Without loss of generality, we assume that $\phi_i$, $1\leq i\leq d$, are linearly independent and are normalized so that $\|\phi(s,a)\|_1\leq 1$ for all $(s,a)$, where $\phi(s,a)=[\phi_1(s,a),\cdots,\phi_d(s,a)]$ is the feature associated with state-action pair $(s,a)$. Let $\Phi=[\phi_1,\cdots,\phi_d]$ be the feature matrix. We parameterize the policy and the $Q$-function using compatible function approximation \citep{sutton1999policy}. In particular, we use softmax parametrization for the policy, i.e., $\pi_\theta(a|s)=\frac{\exp(\phi(s,a)^\top \theta)}{\sum_{a'\in\mathcal{A}}\exp(\phi(s,a')^\top \theta)}$ for all $(s,a)$, where $\theta\in\mathbb{R}^{d}$ is the parameter. As for the $Q$-function, we approximate it from the linear sub-space given by $\mathcal{Q}=\{Q_w=\Phi w\mid w\in\mathbb{R}^d\}$, where $w\in\mathbb{R}^{d}$ is the corresponding parameter. Note that the compatible features in the case of our actor (which utilizes the $Q$-function) are indeed $\{\phi(s,a)\}$. The reason that our features are different than that of \cite{bhatnagar2009natural,sutton1999policy} is because \cite{bhatnagar2009natural,sutton1999policy} use the advantage function in the actor update. When using the advantage function, the corresponding parametric features are $\{\phi(s,a)-\mathbb{E}_{A\sim \pi}[\phi(s,A)]\}$.

By doing linear function approximation, we now only need to work with $d$-dimensional objects (i.e., $w$ for the $Q$-function and $\theta$ for the policy), where $d$ is usually chosen to be much smaller than $|\mathcal{S}||\mathcal{A}|$.

\subsection{Off-Policy Multi-Step TD-learning with Linear Function Approximation}\label{subsec:critic}
In this section, we present the $n$-step off-policy TD-learning algorithm under linear function approximation \citep{sutton2018reinforcement}, which is used for solving the policy evaluation (critic) sub-problem in our AC framework. Let $\pi$ be the target policy we aim to evaluate, and let $\pi_b$ be the behavior policy we used to collect samples. For any state-action pairs $(s,a)$, let $\rho(s,a)=\frac{\pi(a|s)}{\pi_b(a|s)}$, which is called the importance sampling ratio between $\pi$ and $\pi_b$ at $(s,a)$. For any positive integer $n$, Algorithm \ref{algo:critic} presents the off-policy $n$-step TD-learning algorithm for estimating $Q^\pi$.

\begin{algorithm}[h]\caption{Off-Policy $n$-Step TD-Learning with Linear Function Approximation}\label{algo:critic}
	\begin{algorithmic}[1] 
		\STATE {\bfseries Input:} $K$, $\alpha$, $w_0$, $\pi$, $\pi_b$, and $\{(S_k,A_k)\}_{0\leq k\leq (K+n)}$ (a \textit{single trajectory} generated by the behavior policy $\pi_b$)\\
		\FOR{$k=0,1,\cdots,K-1$}
		\STATE $\delta_{k,i}=\mathcal{R}(S_i,A_i)+\gamma \rho(S_{i+1},A_{i+1})\phi(S_{i+1},A_{i+1})^\top w_k-\phi(S_i,A_i)^\top w_k$ 
		\STATE $\Delta_{k,n}=\sum_{i=k}^{k+n-1}\gamma^{i-k}\prod_{j=i+1}^{k+n-1}\rho(S_j,A_j)\delta_{k,i}$
		\STATE $w_{k+1}=w_k+\alpha \phi(S_k,A_k)\Delta_{k,n}$
		\ENDFOR
		\STATE\textbf{Output:} $w_K$
	\end{algorithmic}
\end{algorithm}

In Algorithm \ref{algo:critic}, we employ the importance sampling ratio to account for the discrepancy between the target policy $\pi$ and the behavior policy $\pi_b$.
Although all the three elements of the deadly triad (bootstrapping, function approximation, and off-policy sampling) \citep{sutton2018reinforcement} are present, we show that by choosing $n$ appropriately, Algorithm \ref{algo:critic} has provable finite-sample convergence guarantee. The detailed statement of the result is presented in Section \ref{subsec:bounds}. In this section we provide some intuition. 

Suppose that the Markov chain $\{(S_k,A_k)\}_{k\geq 0}$ under the behavior policy $\pi_b$ has a unique stationary distribution $\kappa_b\in\Delta^{|\mathcal{S}||\mathcal{A}|}$. Let $\kappa_{b,\min}=\min_{s,a}\kappa_b(s,a)$ and let $\mathcal{K}=\text{diag}(\kappa_b)$. Algorithm \ref{algo:critic} can be interpreted as a stochastic approximation (SA) algorithm for solving the equation $\Phi^\top \mathcal{K}(\mathcal{T}_\pi^n(\Phi w)-\Phi w)=0$ as explained in Section \ref{subsec:pf:critic}, which is equivalent to the projected Bellman equation:
\begin{align}\label{eq:pbe}
Q_w=\Pi_{\kappa_b}\mathcal{T}_\pi^n(Q_w)=\Phi (\Phi^\top \mathcal{K}\Phi)^{-1}\Phi^\top \mathcal{K}\mathcal{T}_\pi^n(Q_w).
\end{align}
Here $\mathcal{T}_\pi^n(\cdot)$ denotes the $n$-step Bellman operator, and $\Pi_{\kappa_b}(\cdot)$ stands for the projection operator onto the linear sub-space $\mathcal{Q}$ with respect to the weighted $\ell_2$-norm with weights $\{\kappa_b(s,a)\}_{(s,a)\in\mathcal{S}\times\mathcal{A}}$ \citep{tsitsiklis1997analysis}. It is well-known that the operator $\mathcal{T}_\pi(\cdot)$ (i.e., the one-step Bellman operator) is a contraction mapping\footnote{It is also known that $\mathcal{T}_\pi(\cdot)$ is a contraction mapping with respect to the weighted $\ell_2-$ norm $\|\cdot\|_\kappa$, where $\kappa\in \Delta^{|\mathcal{S}||\mathcal{A}|}$ is the stationary distribution of the Markov chain $\{(S_k,A_k)\}$ under the target policy $\pi$ \citep{tsitsiklis1997analysis}. However, since we do not assume that the target policy induces an ergodic Markov chain, such $\kappa$ may not be unique and/or may not induce a norm. Hence we cannot use this contraction property here.} with respect to $\|\cdot\|_\infty$, with contraction factor $\gamma$. Moreover, the projection operator $\Pi_{\kappa_b}(\cdot)$ is a non-expansive operator with respect to the weighted $\ell_2$-norm $\|\cdot\|_{\kappa_b}$. However, due to the norm mismatch, the composed operator $\Pi_{\kappa_b}\mathcal{T}_\pi(\cdot)$ need not be a contraction mapping with respect to either $\|\cdot\|_\infty$ or $\|\cdot\|_{\kappa_b}$. Specifically, for any given $Q_1$ and $Q_2$, in general we only have
\begin{align}\label{eq:norm_mismatch}
    \|\Pi_{\kappa_b}\mathcal{T}_\pi(Q_1)-\Pi_{\kappa_b}\mathcal{T}_\pi(Q_2)\|_{c}\leq 
    (\gamma/\sqrt{\kappa_{b,\min}}) \|Q_1-Q_2\|_{c},
\end{align}
where $c=\infty$ or $c=\kappa_b$.
In fact, it is not clear if $\Pi_{\kappa_b}\mathcal{T}_\pi(\cdot)$ can be contractive with respect to any norm. This is the fundamental mathematical reason for the divergence of off-policy one-step TD \citep{sutton2018reinforcement}.

Now consider the composed operator $\Pi_{\kappa_b}\mathcal{T}_\pi^n(\cdot)$. Observe that the $n$-step TD operator $\mathcal{T}_\pi^n(\cdot)$ is a contraction mapping with respect to $\|\cdot\|_\infty$, with contraction factor $\gamma^n$. Since the contraction factor of $\mathcal{T}_\pi^n(\cdot)$ decreases geometrically fast as $n$ increases, by choosing $n$ large enough, one can ensure that $\Pi_{\kappa_b}\mathcal{T}_\pi^n(\cdot)$ is a contraction with respect to any chosen norm. This important observation enables us to establish the convergence of Algorithm \ref{algo:critic} in Section \ref{subsec:bounds}. A similar idea was exploited in off-policy TD$(\lambda)$ algorithm in \cite{bertsekas2009projected,yu2012least}, where it was shown that for $\lambda$ close to unity, the off-policy TD$(\lambda)$ algorithm converges. However, \cite{bertsekas2009projected,yu2012least} require an additional projection step in the algorithm to establish the convergence, and no finite-sample guarantees were shown.

In existing literature, to achieve stability in the presence of the deadly triad, algorithms such as GTD \citep{sutton2008convergent}, TDC \citep{sutton2009fast}, and Emphatic TD \citep{sutton2016emphatic} all require to maintain two iterates. Such two time-scale algorithms are in general harder to implement. In addition, the limit point of GTD-type algorithms can only be characterized when the projected Bellman equation (\ref{eq:pbe}) has a unique solution, which is naturally satisfied in the on-policy setting, but stated as an assumption in the off-policy setting; see for example \cite[Theorems $1$ and $2$]{sutton2009fast}. By exploiting multi-step return, Algorithm \ref{algo:critic} naturally achieves convergence, requires to maintain only one iterate, and has a limit point that can be characterized as the solution (which is guaranteed to exist and is unique) of the $n$-step projected Bellman equation.

\subsection{Off-Policy Variant of NAC with Linear Function Approximation}\label{subsec:actor}

In this section, we combine the off-policy TD-learning with linear function approximation algorithm in the previous section, with our variant of NPG update to form the off-policy variant of NAC algorithm. For simplicity of notation, we denote $Q^{\pi_{\theta_t}}$ as $Q^{\pi_t}$. Also, with input $K$, $\alpha$, $w_0$, $\pi$, $\pi_b$, and samples $\{(S_k,A_k)\}_{0\leq k\leq K+n}$, we denote the output of Algorithm \ref{algo:critic} as 
\begin{align*}
    \textsc{critic}(K,\alpha,w_0,\pi,\pi_b,\{S_k,A_k\}_{0\leq k\leq K+n}).
\end{align*}For any integer $T\geq 1$, let $\hat{T}$ be a uniform sample  from $\{0,1,...,T-1\}$. 

\begin{algorithm}[h]
\caption{Off-Policy Natural Actor-Critic with Linear Function Approximation}\label{algo}
	\begin{algorithmic}[1] 
		\STATE {\bfseries Input:} $\hat{T}$, $K$, $\alpha$, $\beta$, $\theta_0$, $\pi_b$, and a \textit{single trajectory} $\{(S_k,A_k)\}_{0\leq k\leq \hat{T}(K+n)}$ generated by $\pi_b$\\
		\FOR{$t=0,1,\dots,\hat{T}-1$}
		\STATE $w_{t}=\textsc{cr}(K,\alpha,\bm{0},\pi_t,\{(S_k,A_k)\}_{t(K+n)\leq k\leq (t+1)(K+n)})$
		\STATE $\theta_{t+1} =\theta_t + \beta w_{t}$
		\ENDFOR
		\STATE\textbf{Output:} $\theta_{\hat{T}}$
	\end{algorithmic}
\end{algorithm}

In each iteration of the off-policy NAC algorithm \ref{algo}, the critic first estimates the $Q$-function $Q^{\pi_{t}}$ using $\Phi w_t$. Then, the actor updates the parameter $\theta_t$ of the current policy. Note that unlike the on-policy NAC where the algorithm usually needs to be constantly reset to a specific state of the environment, which is impractical, off-policy sampling enables us to use a single sample trajectory collected under the behavior policy. 

In existing literature of NAC algorithm with linear function approximation, the critic aims at finding the projection (onto $\mathcal{Q}$) of the target $Q$-function $Q^{\pi_t}$ with respect to a suitable norm involving the state visitation distribution $d^{\pi_t}$  \citep{agarwal2021theory}. More specifically,
$w_t$ is an estimate of the minimizer of the optimization problem 
\begin{align}\label{optimization}
    \mathbb{E}_{s\sim d^{\pi_t},a\sim \pi_t}[(Q^{\pi_t}(s,a)-\phi(s,a)^\top w)^2].
\end{align}
However, the distribution $d^{\pi_t}$ is unknown and also requires special sampling \cite[Section 6]{agarwal2021theory}. Moreover, such sampling requires constant reset of the system, which is necessary in variants of AC algorithms proposed in many related literature; see \cite[Appendix C]{khodadadian2021finite2} for a more detailed discussion.

In the tabular setting, the solution of the optimization problem (\ref{optimization}) is simply the $Q$-function $Q^{\pi_t}$. In the function approximation setting, the solution can be interpreted as an approximation of the $Q$-function $Q^{\pi_t}$ from the chosen linear sub-space. We propose obtaining such approximation by solving the projected Bellman equation, which avoids the use of $d^{\pi_t}$, and enables using a single trajectory of Markovian samples. The projected Bellman equation was introduced in \cite{tsitsiklis1997analysis} for analyzing on-policy TD with linear function approximation. Here we generalize the result of \cite{tsitsiklis1997analysis} to the off-policy setting and we use it in the critic of NAC. 

As an aside, NPG algorithm can be alternatively viewed as a gradient ascent algorithm with the Fisher information matrix as the pre-conditioner. See \cite[Eq. (16)]{agarwal2021theory} for more details. 

\subsection{Finite-Sample Convergence Guarantees}\label{subsec:bounds}

In this section, we present the finite-sample convergence bounds of Algorithms \ref{algo:critic} and \ref{algo}. We begin by stating our one and only assumption.

\begin{assumption}\label{as:MC}
The behavior policy $\pi_b$ satisfies $\pi_b(a|s)>0$ for all $(s,a)$ and the Markov chain $\{S_k\}$ induced by the behavior policy is irreducible and aperiodic.
\end{assumption}

Assumption \ref{as:MC} is standard in studying off-policy TD-learning algorithms \citep{maei2018convergent, zhang2020provably}. Since we work with finite state and action spaces, under Assumption \ref{as:MC}, the Markov chain $\{S_k\}$ admits a unique stationary distribution, denoted by $\mu_b\in\Delta^{|\mathcal{S}|}$ \citep{levin2017markov}. In addition, we have $\|P^k(s,\cdot)-\mu_b(\cdot)\|_{\text{TV}}\leq C\sigma^k$ for any $k\geq 0$, where $C>0$, $\sigma\in (0,1)$ are constants, and $\|\cdot\|_{\text{TV}}$ stands for the total variation distance between probability distributions \citep{levin2017markov}. Note that in this case the random process $\{(S_k,A_k)\}$ is also a Markov chain with a unique stationary distribution, which we have denoted by $\kappa_b\in\Delta^{|\mathcal{S}||\mathcal{A}|}$, and $\kappa_b(s,a)=\mu_b(s)\pi_b(a|s)$ for all $(s,a)$. 

In the existing literature, where on-policy NAC was studied, it is typically required that all the policies achieved in the iterations of the NAC induce ergodic Markov chains over the state-action space \citep{qiu2019finite,wu2020finite}. Such a requirement is strong and not possible to satisfy in an MDP where the optimal policy is a unique  deterministic policy. Off-policy sampling enables us to relax such an unrealistic requirement while also ensuring exploration.

We next present the finite-sample convergence bound of the off-policy TD-learning algorithm \ref{algo:critic} with constant stepsize. The result for using diminishing stepsizes is presented in Appendix \ref{pf:diminishing}. We begin by introducing some notation. For a given stepsize $\alpha$, let $t_\alpha=\min\{k\geq 0:\|P^k(s,\cdot)-\mu_b(\cdot)\|_{\text{TV}}\leq \alpha\}$,  which represents the mixing time of the Markov chain $\{S_k\}$, and can be bounded by an affine function of $\log(1/\alpha)$ under Assumption \ref{as:MC}. Let $f(x)=n+1$ when $x=1$ and $f(x)=\frac{1-x^{n+1}}{1-x}$ when $x\neq 1$. Denote $w_\pi$ as the solution of the projected Bellman equation (\ref{eq:pbe}). Let $\zeta_\pi=\max_{s,a}\frac{\pi(a|s)}{\pi_b(a|s)}$, which measures the mismatch between $\pi$ and $\pi_b$. Let $\lambda_{\min}$ be the smallest eigenvalue of the positive definite matrix $\Phi^\top \mathcal{K}\Phi$.

\begin{theorem}\label{thm:critic}
Consider $\{w_k\}$ of Algorithm \ref{algo:critic}. Suppose that Assumptions \ref{as:MC} is satisfied, the parameter $n$ is chosen such that $n\geq \frac{2\log(\gamma_c)+\log(\kappa_{b,\min})}{2\log(\gamma)}$ (where $\gamma_c\in (0,1)$ is some tunable constant), and $\alpha$ is chosen such that $\alpha (t_\alpha+n+1)\leq \frac{1-\gamma_c}{456 f(\gamma\zeta_\pi)^2}$. Then for all $k\geq t_\alpha+n+1$ we have:
	\begin{align}\label{eq:bound1}
		\mathbb{E}[\|w_k-w_\pi\|_2^2]\leq \underbrace{c_1(1-(1-\gamma_c)\lambda_{\min}\alpha)^{k-(t_\alpha+n+1)}}_{\mathcal{E}_1: \text{ convergence bias}}+\underbrace{c_2\frac{\alpha (t_\alpha+n+1)}{(1-\gamma_c)\lambda_{\min}}}_{\mathcal{E}_2: \text{variance}},
	\end{align}
where $c_1=(\|w_0\|_2+\|w_0-w_\pi\|_2+1)^2$ and $c_2=114f(\gamma\zeta_\pi)^2(\|w_\pi\|_2+1)^2$. Moreover, when the stepsizes satisfy $\sum_{k=0}^\infty\alpha_k=\infty$ and $\sum_{k=0}^\infty\alpha_k^2<\infty$, we have $\lim_{k\rightarrow\infty}w_k=w_\pi$ almost surely. 
\end{theorem}

\begin{remark}
Note that the choice of $n$ here depends on the unknown parameter $\kappa_{b,\min}$, which is a limitation of Theorem \ref{thm:critic}. In implementation, we can first "pretend" that $\kappa_{b}$ is uniform (which implies $\kappa_{b,\min}=1/|\mathcal{S}||\mathcal{A}|$), and initialize $n$ at the value $\frac{2\log(\gamma_c)-\log(|\mathcal{S}||\mathcal{A}|)}{2\log(\gamma)}$. As the algorithm progresses, we keep track of the iterates and see if we detect divergence. If that happens we increase the value of $n$, otherwise we leave $n$ unchanged. 
\end{remark}

As we see from Theorem \ref{thm:critic}, when using constant stepsize in Algorithm \ref{algo:critic}, the convergence bias has geometric rate while the variance is a constant with size $\mathcal{O}(\alpha\log(1/\alpha))$. This phenomenon is well observed in SA literature \citep{srikant2019finite}.

Regarding the choice of the parameter $n$, recall from Section \ref{subsec:critic} that, to ensure the convergence of Algorithm \ref{algo:critic}, we need to choose the parameter $n$ large enough so that $\gamma^n$ is small enough to kill the norm mismatch constant $1/\sqrt{\kappa_{b,\min}}$ (cf. Eq. (\ref{eq:norm_mismatch})). Such a requirement on $n$ is explicitly given in Theorem \ref{thm:critic}. Under that condition, the operator $\Pi_{\kappa_b}\mathcal{T}_\pi^n(\cdot)$ is a contraction mapping with respect to both $\|\cdot\|_\infty$ and $\|\cdot\|_{\kappa_b}$, with a common contraction factor $\gamma_c\in (0,1)$. We make the parameter $\gamma_c$ a tunable constant which can be properly chosen to improve the algorithm performance.

We next present the finite-sample convergence bound of the off-policy NAC with linear function approximation. Let $\xi=\max_{\theta}\|Q^{\pi_\theta}-\Phi w_{\pi_\theta}\|_\infty$, where $Q^{\pi_\theta}$ is the $Q$-function associated with the policy $\pi_\theta$, and $w_{\pi_{\theta}}$ is the solution to the projected Bellman equation $\Phi w=\Pi_{\kappa_b}\mathcal{T}_{\pi_\theta}^n(\Phi w)$. Note that the quantity $\xi$ measures how powerful the function approximation architecture is. Let $\zeta_{\max}=\max_{s,a}\frac{1}{\pi_b(a|s)}$, which is an uniform upper bound of $\zeta_\pi$ for any target policy $\pi$.

\begin{theorem}\label{thm:NAC}
Consider the output $\theta_{\hat{T}}$ of Algorithm \ref{algo}. Under the same assumptions of Theorem \ref{thm:critic}, for any starting distribution $\mu$, we have for any $K\geq t_\alpha+n+1$ and $T\geq 1$:
\begin{align*}
V^{\pi^*}(\mu)&-\mathbb{E}\left[V^{\pi_{\hat{T}}}(\mu)\right]
	\leq  \underbrace{\frac{2}{(1-\gamma)^2T}}_{A_1:\text{ convergence bias in the actor}}+\underbrace{\frac{4\xi}{(1-\gamma)^2}}_{A_2:\text{ bias due to function approximation}}\\
	&+\underbrace{\frac{4}{(1-\gamma)^2}c_3(1-(1-\gamma_c)\lambda_{\min}\alpha)^{\frac{K-(t_\alpha+n+1)}{2}}}_{A_3:\text{ convergence bias in the critic}}+\underbrace{\frac{44 c_3 f(\gamma\zeta_{\max})[\alpha(t_\alpha+n+1)]^{1/2}}{(1-\gamma)^2(1-\gamma_c)^{1/2}\lambda^{1/2}_{\min}}}_{A_4:\text{ variance in the Critic}}.
\end{align*}
Here $c_3=1+\max_{\pi}\|w_\pi\|_2$, where $\max_{\pi}\|w_\pi\|_2\leq  \frac{2}{(1-\gamma_c)^{1/2}(1-\gamma)\sqrt{\lambda_{\min}}}$. 
\end{theorem}

The term $A_1$ represents the convergence bias of the actor, and goes to zero at a rate of $\mathcal{O}(1/T)$ as the outer loop iteration number $T$ goes to infinity. The term $A_3$ measures the convergence bias in the critic, and goes to zero geometrically fast as the inner loop iteration number $K$ goes to infinity. The term $A_4$ represents the impact of the variance in the critic, and is of the size $\mathcal{O}(\sqrt{\alpha\log(1/\alpha)})$, which goes to zero as the inner loop stepsize $\alpha$ goes to zero. 

The term $A_2$ captures the error introduced to the system due to function approximation, and cannot be eliminated asymptotically. Moreover, known results in approximate policy iteration (API) literature suggest that the $1/(1-\gamma)^2$ coefficient inside the term $A_2$ is inevitable. Specifically, it was shown in \cite{bertsekas2011approximate,bertsekas1996neuro} that when $\max_{\pi}\|V^\pi-\Phi w_\pi\|_\infty\leq \xi$, under the API algorithm $\limsup_{k\rightarrow \infty}\|V^{\pi_k}-V^{\pi^*}\|_\infty\leq \frac{2\gamma\xi}{(1-\gamma)^2}$, and an example is presented in \citep[Section 6.2.3]{bertsekas1996neuro}, where the inequality is tight. Since NAC algorithm can be viewed as an API algorithm with a softmax policy update (which is also weighted by the current policy), it is natural to expect a similar function approximation bias. 
Therefore, to improve the function approximation bias term $A_2$, one has to develop instance dependent bound, which is one of our future direction. 

Note that when $A_2=0$ (i.e., when the $Q$-functions corresponding to all the policies in the parametric space are linearly parametrizable), Theorem \ref{thm:NAC} implies convergence to the true optimal policy, which indicates that the optimal policy must also be linearly parametrizable. In fact, suppose we have complete information of the underlying MDP model and were able to implement the general QNPG update (see Appendix \ref{ap:global_convergence} for general QNPG update). Then we have convergence to the global optimal policy. Although this result is a direct implication of Theorem \ref{thm:NAC}, we provide a simpler and more intuitive proof in Appendix \ref{ap:global_convergence}.

To further understand the parameter $\xi$, consider tabular RL, which can be thought of as a special case of RL under linear function approximation with $|\mathcal{S}||\mathcal{A}|$ feature vectors that correspond to the canonical basis vectors, i.e.,  $\Phi$ is an identity matrix. In this special case, Algorithm \ref{algo:critic} and Theorem \ref{thm:critic} give the finite-sample bounds of $n$-step off-policy tabular TD in  \cite[Theorem 2.1]{khodadadian2021finite2}. The actor in Algorithm \ref{algo} reduces to the NPG update \cite[Lemma 5.1 ]{agarwal2021theory}. Furthermore, the function approximation bias $\xi$ in this case is zero, and the finite sample bounds in Theorem \ref{thm:NAC} reduce to the ones presented in \cite[Theorem 2.2]{khodadadian2021finite2}. Compared to \cite{khodadadian2021finite2}, we have an improved dependence on the effective horizon $1/(1-\gamma)$ and the size of the state-action space. The additional factors of $\log^{1/2}|\mathcal{S}||\mathcal{A}|$ and  $1/(1-\gamma)$ in \cite{khodadadian2021finite2} is due to the fact that they were exploiting the $\ell_\infty$-norm contraction of the corresponding variant of the Bellman operator. Here, due to the flexibility in choosing $n$, we are able to exploit the $\ell_2$-norm contraction property, which is "nicer" than $\ell_\infty$-norm contraction. This eventually enables us to remove the additional factor of $\log^{1/2}|\mathcal{S}||\mathcal{A}|$ and $1/(1-\gamma)$ in \cite{khodadadian2021finite2}. See \cite[Lemma 2.1]{chen2021finite} and the paragraph below for more details about the difference between stochastic approximation algorithms under $\ell_2$-norm contraction and $\ell_\infty$-norm contraction.

\subsection{Sample Complexity Analysis}

In this section, we derive sample complexity of off-policy NAC algorithm based on Theorem \ref{thm:NAC}, whose proof is presented in Appendix \ref{pf:co:sc}.

\begin{corollary}\label{co:sc}
In order to achieve $V^{\pi^*}(\mu)-\mathbb{E}\left[V^{\pi_{\hat{T}}}(\mu)\right]\leq \epsilon+\frac{4\xi}{(1-\gamma)^2}$, the number of samples requires is of the size 
\begin{align*}
    \mathcal{O}\left(\epsilon^{-3}\log^2(1/\epsilon)\right)\Tilde{\mathcal{O}}\left(f(\gamma\zeta_{\max})^2 n(1-\gamma)^{-8}(1-\gamma_c)^{-3}\lambda_{\min}^{-3}\right).
\end{align*}
\end{corollary}

\begin{remark}
It was argued in \citep[Appendix C]{khodadadian2021finite2} that sample complexity is not well-defined when the convergence error does not go to zero.  Therefore, one should not use sample complexity when we do not have global convergence due to the function approximation bias. However, we present Corollary \ref{co:sc} in terms of ``sample complexity'' in the same sense as used in prior literature to enable a fair comparison. 
See Appendix \ref{ap:sample_complexity} for a more  detailed discussion.
\end{remark}

In view of the sample complexity bound, the dependency on the required accuracy level $\epsilon$ is $\Tilde{\mathcal{O}}(\epsilon^{-3})$. This improves the state-of-the-art sample complexity of off-policy NAC with function approximation result in the literature by a factor of $\epsilon^{-1}$ (cf. Table \ref{table: results2}). Observe that the tunable constant $\gamma_c$ appears as $(1-\gamma_c)^{-3}$ in the bound. This makes intuitive sense in that $\gamma_c$ is the effective contraction ratio of the composed operator $\Pi_{\kappa_b}\mathcal{T}_\pi^n(\cdot)$ in the critic. Hence we expect better sample complexity for smaller $\gamma_c$. As stated in Theorem \ref{thm:critic}, in order to use smaller $\gamma_c$ in our analysis, we need to choose larger $n$ in executing Algorithm \ref{algo:critic}. An advantage of using large $n$ is that it leads to a lower function approximation bias $\xi$. To see this, consider the projected Bellman equation \eqref{eq:pbe}. When $n$ tends to infinity, since $\lim_{n\rightarrow\infty}\mathcal{T}_\pi^n(\Phi w)=Q^\pi$ due to value iteration (Banach fixed-point theorem for the operator $\mathcal{T}_\pi(\cdot)$), the solution of the projected Bellman equation coincides with the projection of $Q^\pi$ to the linear sub-space $\mathcal{Q}$, which has the best function approximation bias. However, note that the parameter $n$ also appears in the numerator of the sample complexity bound (which is due to the variance term in the critic), hence there is a trade-off in the choice of $n$.  To summarize, increasing (decreasing) the parameter $n$ leads to better (worse) critic convergence bias and function approximation bias, but has worse (better) critic variance.

In general, the issue of high variance due to the importance sampling ratio (cf. $\zeta_{\max}$) is a fundamental problem in multi-step off-policy TD-learning \citep{sutton2018reinforcement}. In order to reduce such high variance, several variants of off-policy RL such as Retrace$(\lambda)$ \citep{munos2016safe}, $V$-trace \citep{espeholt2018impala}, and $Q$-trace \citep{khodadadian2021finite2} have been proposed. These algorithms use truncated importance sampling ratios to reduce $\zeta_{\max}$, thus reducing the variance. However, none of them are shown to converge in the function approximation setting. Designing efficient algorithms to control the high variance in multi-step off-policy TD-learning with function approximation is one of our future directions.

\section{Proof Sketch of Theorems \ref{thm:critic} and \ref{thm:NAC}}\label{sec:proof}

In this section, we present the proof sketch of Theorems \ref{thm:critic} and \ref{thm:NAC}. The detailed proof is presented in Appendices \ref{ap:critic} and \ref{ap:actor}, respectively.

\subsection{Proof Sketch of Theorem \ref{thm:critic}}\label{subsec:pf:critic}

We begin by remodeling the update equation of Algorithm \ref{algo:critic} (line 4) as a Markovian SA algorithm. For any $k\geq 0$, let $X_k=(S_k,A_k,...,S_{k+n},A_{k+n})$, which is a Markov chain. Denote the state space of $\{X_k\}$ by $\mathcal{X}$. Note that $\mathcal{X}$ is finite. Define an operator $F:\mathbb{R}^{d}\times\mathcal{X}\mapsto\mathbb{R}^{d}$ by
\begin{align*}
    &F(w,x)=F(w,s_0,a_0,...,s_n,a_n)\\
    =\;&\phi(s_0,a_0)\sum_{i=0}^{n-1}\gamma^i\prod_{j=1}^{i}\rho(s_j,a_j)\left(\mathcal{R}(s_i,a_i)+\gamma \rho(s_{i+1},a_{i+1})\phi(s_{i+1},a_{i+1})^\top w-\phi(s_i,a_i)^\top w\right).
\end{align*}

\vspace{-3mm}
Then the update equation of Algorithm \ref{algo:critic} can be equivalently written as
\begin{align}\label{algo:sa:remodel}
    w_{k+1}=w_k+\alpha F(w_k,S_k,A_k,...,S_{k+n},A_{k+n}).
\end{align}
Define the expected operator $\Bar{F}:\mathbb{R}^d\mapsto\mathbb{R}^d$ of $F(\cdot)$ by $\Bar{F}(w)=\mathbb{E}_{S_0\sim \mu_b}[F(w,S_0,A_0,...,S_n,A_n)]$. Then Algorithm (\ref{algo:sa:remodel}) can be viewed as a Markovian SA algorithm for solving the equation $\Bar{F}(w)=0$. 

To proceed and establish finite-sample bound of Algorithm (\ref{algo:sa:remodel}), we will apply Markovian SA results in the literature. In particular, we will apply Theorem 2.1 of \cite{chen2019finitesample}, which is presented in Appendix \ref{pf:thm:critic} for self-containedness. To achieve that, we establish properties of the operators $F(\cdot)$, $\bar{F}(\cdot)$, and the Markov chain $\{Y_k\}$ in the following proposition, which guarantee that all the assumptions for applying \cite[Theorem 2.1]{chen2019finitesample} is satisfied. The proof is presented in Appendix \ref{pf:prop:critic}.

\begin{proposition}\label{prop:critic}
Suppose Assumption \ref{as:MC} is satisfied and $n\geq \frac{2\log(\gamma_c)+\log(\kappa_{b,\min})}{2\log(\gamma)}$. 
	\begin{enumerate}[(1)]
		\item 
		The operator $F(w,x)$ satisfies $\|F(w_1,x)-F(w_2,x)\|_2\leq 2f(\gamma\zeta_\pi)	\|w_1-w_2\|_2$ and $\|F(\bm{0},x)\|_2\leq f(\gamma\zeta_\pi)$ for any $w_1,w_2\in\mathbb{R}^d$ and $x\in\mathcal{X}$.
		\item The Markov chain $\{X_k\}$ has a unique stationary distribution, denoted by $\nu_b$. Moreover, it holds for any $k\geq 0$ that $\max_{x\in\mathcal{X}}\left\|P^{k+n+1}(x,\cdot)-\nu_b(\cdot)\right\|_{\text{TV}}\leq C\sigma^k$,
		where the constants $C$ and $\sigma$ are given right after Assumption \ref{as:MC}.
		\item 
		\begin{enumerate}[(a)]
			\item $\bar{F}(w)$ is explicitly given by $\bar{F}(w)=\Phi^\top \mathcal{K} (\mathcal{T}_\pi^n(\Phi w)-\Phi w)$.
			\item $\bar{F}(w)=0$ has a unique solution, which we have denoted by $w_\pi\in\mathbb{R}^d$.
			\item Let $M(w)=\frac{1}{2}\|w\|_2^2$. Then we have $\langle\nabla M(w-w_\pi),\bar{F}(w)\rangle\leq -2(1-\gamma_c)\lambda_{\min}M(w-w_\pi)$ for any $w\in\mathbb{R}^d$.
		\end{enumerate}
	\end{enumerate}
\end{proposition}

Proposition \ref{prop:critic} (1) states that the operator $F(w,x)$ is Lipschitz in terms of $w$, which further implies affine growth rate of $F(w,x)$ in the sense that $\|F(w,x)\|_2\leq f(\gamma\zeta_\pi)(\|w\|_2+1)$ for any $w\in\mathbb{R}^d$ and $x\in\mathcal{X}$. Proposition \ref{prop:critic} (2) states that the auxiliary Markov chain $\{X_k\}$ also preserves the geometric mixing property, which is particularly useful for us to control the Markovian noise in the update equation (\ref{algo:sa:remodel}). Proposition \ref{prop:critic} (3) implies that using $M(w-w_\pi)$ as the Lyapunov function, both SA algorithm (\ref{algo:sa:remodel}) and its associated ODE have a negative drift. This is the key property used in \cite{chen2019finitesample} to establish the finite-sample convergence bounds. Now we are ready to apply \citep[Theorem 2.1]{chen2019finitesample} to establish finite-sample bounds of Algorithm (\ref{algo:sa:remodel}) (and hence Algorithm \ref{algo:critic}). The details are presented in Appendix \ref{pf:thm:critic}.

\subsection{Proof Sketch of Theorem \ref{thm:NAC}}

First, we show an equivalent form of the update equation of the actor parameter $\theta_t$ (line 4 of Algorithm \ref{algo}) in the following lemma. The proof is provided in Appendix \ref{pf:le:rewrite}.

\begin{lemma}\label{le:rewrite}
For any $w,\theta\in\mathbb{R}^d$, let $\theta'=\theta+\beta w$. Then the following relation holds:
\begin{align}\label{eq:equivalent_update}
    \pi_{\theta'}(a|s)=\pi_\theta(a|s)\frac{\exp(\beta w^\top \phi(s,a))}{\sum_{a'\in\mathcal{A}}\pi_{\theta}(a'|s)\exp(\beta w^\top \phi(s,a'))}.
\end{align}
\end{lemma}

Such equivalent update rule is established in \cite{agarwal2021theory} but only under the condition that $w$ is the solution of an appropriate optimization problem, which forbids \cite{agarwal2021theory} from using the equivalent update equation \eqref{eq:equivalent_update} in the analysis of function approximation. Here we establish such equivalence in the case where $w$ is arbitrary. On the one hand, this seemingly simple but important extension enables us to use the lower dimensional update equation $\theta_{t+1}=\theta_t+\beta w_t$ in the algorithm. On the other hand, we can use the equivalent update equation \eqref{eq:equivalent_update} for the analysis to obtain better convergence rate than \cite{agarwal2021theory}. Using Lemma \ref{le:rewrite}, we have the following performance bound for the actor. See Appendix \ref{pf:prop:actor_bound} for the proof.

\begin{proposition}\label{prop:actor_bound}
Consider $\pi_{\hat{T}}$ generated by Algorithm \ref{algo}. Let $\beta=\log(|\mathcal{A}|)$. Then we have
\begin{align}\label{eq:100}
	V^{\pi^*}(\mu)&-\mathbb{E}\left[V^{\pi_{\hat{T}}}(\mu)\right]
	\leq\; \frac{2}{(1-\gamma)^2T}+\frac{4}{(1-\gamma)^2T}\sum_{t=0}^{T-1}\mathbb{E}[\|Q^{\pi_t}-\Phi w_t\|_\infty].
\end{align}
\end{proposition}
\vspace{-3mm}
The first term on the RHS of Eq. \eqref{eq:100} represents the convergence rate of the actor while the second term is a combination of the error in the critic estimate and the function approximation bias. This already improves the result in \cite{agarwal2021theory}, where they have $\mathcal{O}(1/\sqrt{T})$ convergence rate of the actor in the function approximation setting while we have $\mathcal{O}(1/T)$. Note that the $\mathcal{O}(1/T)$ convergence rate matches with the convergence rate of the actor in the tabular setting \citep{agarwal2021theory,khodadadian2021finite,khodadadian2021finite2}.

The last step is to control $\mathbb{E}[\|Q^{\pi_t}-\Phi w_t\|_\infty]$. We first use triangle inequality to obtain
\begin{align}
    \mathbb{E}[\|Q^{\pi_t}-\Phi w_t\|_\infty]
    &\leq \mathbb{E}[\|Q^{\pi_t}-\Phi w_{\pi_t}\|_\infty]+\mathbb{E}[\|\Phi w_{\pi_t}-\Phi w_t\|_\infty]\nonumber\\
    &\leq \mathbb{E}[\|Q^{\pi_t}-\Phi w_{\pi_t}\|_\infty]+\mathbb{E}[\|w_{\pi_t}- w_t\|_\infty],\label{eq:triangle}
\end{align}
where we used $\|\Phi\|_\infty= \max_{s,a}\|\phi(s,a)\|_1\leq 1$.  Observe that the first term on the RHS of Eq. (\ref{eq:triangle}) can be bounded by $\xi$, and the second term can be bounded by applying Theorem \ref{thm:critic} in conjunction with Jensen's inequality. The result then follows from substituting the upper bound of the term $\mathbb{E}[\|Q^{\pi_t}-\Phi w_t\|_\infty]$ into Eq. (\ref{eq:100}) of Proposition \ref{prop:actor_bound}. 

\section{Conclusion}\label{sec:conclusion}
In this paper, we establish finite-sample convergence guarantees of off-policy NAC with linear function approximation. To overcome the deadly triad in the critic, we use $n$-step TD-learning, which is a one-time scale algorithm for policy evaluation using off-policy sampling and linear function approximation, and has provable convergence bounds. As for the analysis of the actor, we identify an equivalent update equation, and use it to conduct refined analysis compared to \cite{agarwal2021theory}. As a result, our finite-sample bounds imply a sample complexity of $\Tilde{\mathcal{O}}(\epsilon^{-3})$, which advances the state-of-the-art result in the literature.

\bibliographystyle{imsart-nameyear}
\bibliography{references}    

\begin{appendix}
\section{Analysis of the Critic}\label{ap:critic}
\subsection{Proof of Proposition \ref{prop:critic}}\label{pf:prop:critic}
\begin{enumerate}[(1)]
\item Let $w_1,w_2\in\mathbb{R}^d$ and $x=(s_0,a_0,...,s_n,a_n)\in\mathcal{X}$ be arbitrary. For simplicity of notation, we denote $\rho_{i,j}=\prod_{k=i}^j\rho(s_k,a_k)$. Then we have
\begin{align*}
	&\|F(w_1,x)-F(w_2,x)\|_2\\
	=\;&\left\|\phi(s_0,a_0)\sum_{i=0}^{n-1}\gamma^{i}(\gamma \rho_{1,i+1} \phi(s_{i+1},a_{i+1})^\top -\rho_{1,i}\phi(s_i,a_i)^\top)(w_1-w_2)\right\|_2\\
	=\;&\left\|\phi(s_0,a_0)(\gamma^n\rho_{1,n}\phi(s_n,a_n)^\top -\phi(s_0,a_0)^\top)(w_1-w_2)\right\|_2\\	
	\leq \;&\|\phi(s_0,a_0)\|_2((\gamma\zeta_\pi)^n\|\phi(s_n,a_n)\|_2 +\|\phi(s_0,a_0)\|_2)\|w_1-w_2\|_2\\
	\leq \; &((\gamma\zeta_\pi)^n+1)\|w_1-w_2\|_2\tag{$\|\phi(s,a)\|_2\leq \|\phi(s,a)\|_1\leq 1$ for all $(s,a)$}\\
	\leq \;& \sum_{i=0}^{n}(\gamma\zeta_\pi)^i\|w_1-w_2\|_2\\
	= \;& f(\gamma\zeta_\pi)\|w_1-w_2\|_2.
\end{align*}

Similarly, we have 
\begin{align*}
\|F(\bm{0},x)\|_2&=\left\|\phi(s_0,a_0)\sum_{i=0}^{n-1}\gamma^{i}\rho_{1,i}\mathcal{R}(s_i,a_i)\right\|_2\\
&\leq \|\phi(s_0,a_0)\|_2\sum_{i=0}^{n-1}(\gamma\zeta_\pi)^{i}|\mathcal{R}(s_i,a_i)|\\
&\leq \sum_{i=0}^{n-1}(\gamma\zeta_\pi)^{i}\\
&\leq f(\gamma\zeta_\pi).
\end{align*}
\item The claim that $\{X_k\}$ has a stationary distribution $\nu_b$ follows directly from its definition and Assumption \ref{as:MC}. Now for any $x=(s_0,a_0,...,s_n,a_n)\in\mathcal{X}$, using the definition of total variation distance, we have for any $k\geq 0$:
\begin{align*}
&\left\|P^{k+n+1}(x,\cdot)-\nu_b(\cdot)\right\|_{\text{TV}}\\
=\;&\frac{1}{2}\sum_{s_0',a_0',\cdots,s_n',a_n'}\left|\sum_{s}P_{a_n}(s_n,s)P^k_{\pi_b}(s,s_0')\!-\!\mu_b(s_0')\right|\!\left[\prod_{i=0}^{n-1}\pi(a_i'\mid s_i')P_{a_i'}(s_i',s_{i+1}')\right]\!\pi(a_n'\mid s_n')\\
=\;&\frac{1}{2}\sum_{s_0'}\left|\sum_{s}P_{a_n}(s_n,s)P^k_{\pi_b}(s_n,s_0')-\mu_b(s_0')\right|\\
\leq \;&\frac{1}{2}\sum_{s}P_{a_n}(s_n,s)\sum_{s_0'}\left|P^k_{\pi_b}(s_n,s_0')-\mu_b(s_0')\right|\\
\leq \;&\max_{s\in\mathcal{S}}\|P^k_{\pi_b}(s,\cdot)-\mu_b(\cdot)\|_{\text{TV}}\\
\leq \;&C\sigma^k.
\end{align*}
It follows that $\max_{x\in\mathcal{X}}\left\|P^{k+n+1}_{\pi_b}(x,\cdot)-\nu_b(\cdot)\right\|_{\text{TV}}\leq C\sigma^k$ for all $k\geq 0$.
\item 
\begin{enumerate}[(a)]
\item We first compute $\bar{F}(w)$. By definition, we have
\begin{align*}
	\bar{F}(w)
	\!=\!\mathbb{E}_{S_0\sim\mu_b}\left[\phi(S_0,A_0)\left(\sum_{i=0}^{n-1}\gamma^{i}\rho_{1,i}\mathcal{R}(S_i,A_i)\!+\!\gamma^n \rho_{1,n}\phi(S_{n},A_{n})^\top w\!-\!\phi(S_0,A_0)^\top w\right)\right].
\end{align*}
Using conditional expectation and the Markov property, we have for any $i=0,...,n-1$:
\begin{align*}
    &\mathbb{E}_{S_0\sim\mu_b}\left[\phi(S_0,A_0)\gamma^{i}\rho_{1,i}\mathcal{R}(S_i,A_i)\right]\\
    =\;&\mathbb{E}_{S_0\sim\mu_b}\left[\phi(S_0,A_0)\gamma^{i}\rho_{1,i-1}\mathbb{E}\left[\rho_i\mathcal{R}(S_i,A_i)\mid S_0,A_0,\dots,S_{i-1},A_{i-1}\right]\right]\\
    =\;&\mathbb{E}_{S_0\sim\mu_b}\left[\phi(S_0,A_0)\gamma^{i}\rho_{1,i-1}\sum_{s,a}P_{A_{i-1}}(S_{i-1},s)\pi_b(a|s)\frac{\pi(a|s)}{\pi_b(a|s)}\mathcal{R}(s,a)\right]\\
    =\;&\mathbb{E}_{S_0\sim\mu_b}\left[\phi(S_0,A_0)\gamma^{i}\rho_{1,i-1}[P_\pi R](S_{i-1},A_{i-1})\right]\\
    =\;&\mathbb{E}_{S_0\sim\mu_b}\left[\phi(S_0,A_0)\gamma^{i}[P_\pi^{i}R](S_0,A_0)\right]\\
    =\;&\Phi^\top \mathcal{K} (\gamma P_\pi)^{i}R,
\end{align*}
where $P_\pi$ is the transition probability matrix of the Markov chain $\{(S_k,A_k)\}$ under policy $\pi$, and $R$ is the reward vector. Similarly, we have
\begin{align*}
    \mathbb{E}_{S_0\sim\mu_b}\left[\phi(S_0,A_0)\gamma^n \rho_{1,n}\phi(S_{n},A_{n})^\top w\right]=\Phi^\top \mathcal{K}(\gamma P_\pi)^n\Phi w,
\end{align*}

and
\begin{align*}
    \mathbb{E}_{S_0\sim\mu_b}\left[\phi(S_0,A_0)\phi(S_0,A_0)^\top w\right]=\Phi^\top \mathcal{K}\Phi w.
\end{align*}
Therefore, we obtain
\begin{align*}
    \bar{F}(w)&=\Phi^\top \mathcal{K} \sum_{i=0}^{n-1}(\gamma P_\pi)^{i}R+\Phi^\top \mathcal{K}(\gamma P_\pi)^n\Phi w-\Phi^\top \mathcal{K}\Phi w\\
    &=\Phi^\top \mathcal{K} \left[\sum_{i=0}^{n-1}(\gamma P_\pi)^{i}R+(\gamma P_\pi)^n\Phi w-\Phi w\right]\\
    &=\Phi^\top \mathcal{K}(\mathcal{T}_\pi^n(\Phi w)-\Phi w).
\end{align*}
\item Note that the equation $\bar{F}(w)=0$ is equivalent to
\begin{align*}
	\Phi w=\Phi(\Phi^\top \mathcal{K}\Phi)^{-1}\Phi^\top \mathcal{K}\mathcal{T}_\pi^n(\Phi w)=\Pi_{\kappa_b}\mathcal{T}_\pi^n(\Phi w),
\end{align*}
which is the projected $n$-step Bellman equation (\ref{eq:pbe}). Observe that
\begin{align*}
    \|\Pi_{\kappa_b}\mathcal{T}_\pi^n(Q_1)-\Pi_{\kappa_b}\mathcal{T}_\pi^n(Q_2)\|_{\kappa_b}&\leq \|\mathcal{T}_\pi^n(Q_1)-\mathcal{T}_\pi^n(Q_2)\|_{\kappa_b}\tag{$\Pi_{\kappa_b}(\cdot)$ is non expansive}\\
    &\leq \|\mathcal{T}_\pi^n(Q_1)-\mathcal{T}_\pi^n(Q_2)\|_{\infty}\tag{norm inequality}\\
    &\leq\gamma^n \|Q_1-Q_2\|_{\infty}\tag{$\mathcal{T}_\pi^n$ is $\gamma^n$-contraction}\\
    &\leq\frac{\gamma^n}{\sqrt{\kappa_{b,\min}}} \|Q_1-Q_2\|_{\kappa_b}\tag{norm inequality}\\
    &\leq \frac{\gamma_c \sqrt{\kappa_{b,\min}}}{\sqrt{\kappa_{b,\min}}} \|Q_1-Q_2\|_{\kappa_b}\tag{requirement on $n$}\\
    &=\gamma_c\|Q_1-Q_2\|_{\kappa_b}.
\end{align*}

It follows that the composed operator $\Pi_{\kappa_b}\mathcal{T}_\pi^n(\cdot)$ is a contraction mapping with respect to $\|\cdot\|_{\kappa_b}$. Therefore, Banach fixed-point theorem implies that the projected Bellman equation \eqref{eq:pbe} has a unique solution. Since the matrix $\Phi$ is full-column rank, there is a unique solution (which we have denoted by $w_\pi$) to the equation $\bar{F}(w)=0$.
	\item Consider the Lyapunov function $M(w)=\frac{1}{2}\|w\|_2^2$. Since the $n$-step Bellman operator $\mathcal{T}_\pi^n(\cdot)$ is linear, we have
	\begin{align*}
		&\langle \nabla M(w-w_\pi),\bar{F}(w)\rangle\\
		=\;&\langle w-w_\pi,\Phi^\top \mathcal{K} (\mathcal{T}_\pi^n(\Phi w)-\Phi w)\rangle\\
		=\;&\langle  (w-w_\pi),\Phi^\top \mathcal{K}\mathcal{T}_\pi^n(\Phi(w-w_\pi))-\Phi^\top \mathcal{K}\Phi(w-w_\pi)\rangle\tag{$\bar{F}(w_\pi)=0$}\\
		=\;&\langle  (\Phi^\top \mathcal{K}\Phi)(w-w_\pi),(\Phi^\top \mathcal{K}\Phi)^{-1}\Phi^\top \mathcal{K}\mathcal{T}_\pi^n(\Phi(w-w_\pi))-(w-w_\pi)\rangle\\
		=\;&\langle  (\Phi^\top \mathcal{K}\Phi)(w-w_\pi),(\Phi^\top \mathcal{K}\Phi)^{-1}\Phi^\top \mathcal{K}\mathcal{T}_\pi^n(\Phi(w-w_\pi))\rangle-\|\Phi(w-w_\pi)\|_{\kappa_b}^2\\
		=\;&\langle  \mathcal{K}^{1/2}\Phi(w-w_\pi),\mathcal{K}^{1/2}\Phi(\Phi^\top \mathcal{K}\Phi)^{-1}\Phi^\top \mathcal{K}\mathcal{T}_\pi^n(\Phi(w-w_\pi))\rangle-\|\Phi(w-w_\pi)\|_{\kappa_b}^2\\
		\leq\;& \|\Phi(w-w_\pi)\|_{\kappa_b}\|\Phi(\Phi^\top \mathcal{K}\Phi)^{-1}\Phi^\top \mathcal{K}\mathcal{T}_\pi^n(\Phi(w-w_\pi))\|_{\kappa_b}-\|\Phi(w-w_\pi)\|_{\kappa_b}^2\tag{Cauchy Schwarz Inequality}\\
		\leq\;& \gamma_c\|\Phi(w-w_\pi)\|_{\kappa_b}\|\Phi(w-w_\pi)\|_{\kappa_b}-\|\Phi(w-w_\pi)\|_{\kappa_b}^2\\
		=\;&-(1-\gamma_c)\|\Phi(w-w_\pi)\|_{\kappa_b}^2\\
		\leq \;&-2(1-\gamma_c)\lambda_{\min}M(w-w_\pi),
	\end{align*}
	where in the last line we used $\sqrt{\lambda_{\min}}\|w\|_2\leq  \|\Phi w\|_{\kappa_b}$ for any $w\in\mathbb{R}^d$.
\end{enumerate}
\end{enumerate}

\subsection{Proof of Theorem \ref{thm:critic}}\label{pf:thm:critic}

Since Algorithm \ref{algo:critic} is a linear stochastic approximation algorithm under Markovian noise. Proposition \ref{prop:critic} ensures the applicability of \cite[Proposition 4.8]{bertsekas1996neuro}, which gives us the almost sure convergence result under nun-summable but squared-summable stepsizes. We next focus on the finite-sample guarantees.

We begin by restating \citep[Corollary 2.1]{chen2019finitesample} in the following, where we adopt our notation for consistency.

\begin{theorem}[Corollary 2.1 of \cite{chen2019finitesample}]\label{thm:chen2019}
Consider the stochastic approximation algorithm
\begin{align*}
    w_{k+1}=w_k+\alpha G(X_k,w_k).
\end{align*}
Suppose that
\begin{enumerate}[(1)]
    \item The random process $\{X_k\}$ has a unique stationary distribution $\nu$, and it holds for any $k\geq  0$ that $\max_{x\in\mathcal{X}}\|P^k(x,\cdot)-\mu(\cdot)\|_{\text{TV}}\leq C_1\sigma_1^k$ for some constant $C_1>0$ and $\sigma_1\in (0,1)$.
    \item The operator $G(\cdot,\cdot)$ satisfies $\|G(x,w_1)-G(x,w_2)\|_2\leq L\|w_1-w_2\|_2$ and $\|G(x,\bm{0})\|_2\leq L$ for any $w_1,w_2\in\mathbb{R}^d$ and $x\in\mathcal{X}$.
    \item The equation $\bar{G}(w)=\mathbb{E}_{X\sim \mu}[G(X,w)]=0$ has a unique solution $w^*$, and the following inequality holds for all $w\in\mathbb{R}^d$: $(w-w^*)^\top \bar{G}(w)\leq -\ell \|w-w^*\|_2^2$, where $\ell>0$ is some positive constant.
    \item The stepsize $\alpha$ is chosen such that $\alpha \tau_\alpha\leq \frac{\ell}{114L^2}$, where 
    \begin{align*}
        \tau_\alpha:=\min\{k\geq 0\;:\;\max_{x\in\mathcal{X}}\|P^k(x,\cdot)-\nu(\cdot)\|_{\text{TV}}\leq\alpha\}.
    \end{align*}
\end{enumerate}
Then we have for any $k\geq \tau_\alpha$ that
\begin{align*}
    \mathbb{E}[\|w_k-w^*\|_2^2]\leq (\|w_0\|_2+\|w_0-w^*\|_2+1)^2(1-\ell \alpha)^{k-\tau_\alpha}+114L^2(\|w^*\|_2+1)^2\frac{\alpha \tau_\alpha}{\ell}.
\end{align*}
\end{theorem}

Now we proceed to prove Theorem \ref{thm:critic}. To apply Theorem \ref{thm:chen2019}, we begin by identifying the corresponding constants using Proposition \ref{prop:critic}. We have
\begin{align*}
    L=f(\gamma\zeta_\pi),\;\ell=(1-\gamma_c)\lambda_{\min},\;\text{ and }\;\tau_\alpha=t_\alpha+n+1.
\end{align*}
It follows that when the constant stepsize $\alpha$ within Algorithm \ref{algo:critic} is chosen such that $\alpha (t_\alpha+n+1)\leq \frac{(1-\gamma_c)\lambda_{\min}}{114f(\gamma\zeta_\pi)^2}$, we have for all $k\geq t_\alpha+n+1$:
\begin{align*}
    \mathbb{E}[\|w_k-w_\pi\|_2^2]\leq c_1(1-(1-\gamma_c)\lambda_{\min} \alpha)^{k-(t_\alpha+n+1)}+\frac{c_2\alpha (t_\alpha+n+1)}{(1-\gamma_c)\lambda_{\min}},
\end{align*}
where $c_1=(\|w_0\|_2+\|w_0-w_\pi\|_2+1)^2$ and $c_2=114f(\gamma\zeta_\pi)^2(\|w_\pi\|_2+1)^2$. This proves Theorem  \ref{thm:critic}.

\subsection{Finite-Sample Bound for Using Diminishing Stepsizes}\label{pf:diminishing}

We here state the finite-sample bounds of Algorithm \ref{algo:critic} for using diminishing stepsizes of the form $\alpha_k=\frac{\alpha}{(k+h)^\eta}$, where $\alpha,h>0$ and $\eta\in (0,1]$. For simplicity of notation, let $t_k=t_{\alpha_k}$,  $L_1=\frac{1+\log(C/\sigma)}{\log(1/\sigma)}$, and $\ell=(1-\gamma_c)\lambda_{\min}$.

\begin{theorem}\label{thm:diminishing}
Consider $\{w_k\}$ of Algorithm \ref{algo:critic}. Suppose that Assumptions \ref{as:MC} is satisfied, the parameter $n$ is chosen such that $n\geq \frac{2\log(\gamma_c)+\log(\kappa_{b,\min})}{2\log(\gamma)}$ (where $\gamma_c\in (0,1)$ is some tunable constant), and $\alpha_k=\frac{\alpha}{(k+h)^\eta}$, where $\alpha>0$, $\eta\in (0,1]$, and $h$ is chosen such that $\sum_{i=k-t_k}^{k-1}(t_i+n+1)\leq \frac{1-\gamma_c}{114 f(\gamma\zeta_\pi)^2}$. Let $\hat{k}:=\min\{k:k\geq  t_k+n+1\}$. Then we have the following results.
\begin{enumerate}[(1)]
	\item When $\eta=1$, we have for all $k\geq \hat{k}$:
	\begin{align*}
		\mathbb{E}[\|w_k-w^*\|_2^2]\leq
		\begin{dcases}
			c_1\left(\frac{\hat{k}+h}{k+h}\right)^{\ell\alpha}+\frac{8c_2\alpha^2L_1}{1-\ell\alpha}\frac{[\log\left(\frac{k+h}{\alpha}\right)+1]}{(k+h)^{\ell\alpha}},& \ell\alpha\in (0,1),\\
			c_1\left(\frac{\hat{k}+h}{k+h}\right)+8c_2\alpha^2L_1\frac{\log(\frac{k+h}{\hat{k}+h})[\log\left(\frac{k+h}{\alpha}\right)+1]}{k+h},& \ell\alpha=1,\\
			c_1\left(\frac{\hat{k}+h}{k+h}\right)^{\ell\alpha}+\frac{8ec_2\alpha^2L_1}{\ell\alpha-1} \frac{\left[\log\left(\frac{k+h}{\alpha}\right)+1\right]}{k+h},& \ell\alpha\in (1,\infty).
		\end{dcases}
	\end{align*}
	\item When $\eta\in (0,1)$ and $\alpha>0$, suppose in addition that $\hat{k}+h\geq [2\eta/(\ell\alpha)]^{1/(1-\eta)}$, then we have for all $k\geq \hat{k}$:
	\begin{align*}
		\mathbb{E}[\|\theta_k-\theta^*\|^2]
		\leq c_1\exp\left[-\frac{\ell\alpha}{1-\eta}\left((k+h)^{1-\eta}-(\hat{k}+h)^{1-\eta}\right)\right]+\frac{4c_2\alpha^2L_1}{\ell\alpha}\frac{[\log\left(\frac{k+h}{\alpha}\right)+1]}{(k+h)^\eta}.
	\end{align*}
\end{enumerate}
\end{theorem}
Similar to Theorem \ref{thm:critic} following from \cite[Corollary 2.1]{chen2019finitesample}, Theorem \ref{thm:diminishing} follows from \citep[Corollary 2.2]{chen2019finitesample}. Hence we omit the proof.

\section{Analysis of the Actor}\label{ap:actor}
\subsection{Proof of Lemma \ref{le:rewrite}}\label{pf:le:rewrite}
Let $\pi$ and $\pi'$ be two policies parametrized by $\theta$ and $\theta'$, respectively. Then we have
\begin{align*}
    \pi'(a|s)&=\frac{\exp(\theta'^\top \phi(s,a))}{\sum_{a'\in\mathcal{A}}\exp(\theta'^\top \phi(s,a'))}\\
    &=\frac{\exp((\theta+\beta w)^\top \phi(s,a))}{\sum_{a'\in\mathcal{A}}\exp((\theta+\beta w)^\top \phi(s,a'))}\\
    &=\frac{\exp(\theta^\top \phi(s,a))\exp(\beta w^\top \phi(s,a))}{\sum_{a'\in\mathcal{A}}\exp((\theta+\beta w)^\top \phi(s,a'))}\\
    &=\frac{\exp(\theta^\top \phi(s,a))}{\sum_{a'\in\mathcal{A}}\exp(\theta^\top \phi(s,a'))}\frac{\exp(\beta w^\top \phi(s,a))\sum_{a'\in\mathcal{A}}\exp(\theta^\top \phi(s,a'))}{\sum_{a'\in\mathcal{A}}\exp((\theta+\beta w)^\top \phi(s,a'))}\\
    &=\pi(a|s)\frac{\exp(\beta w^\top \phi(s,a))\sum_{a'\in\mathcal{A}}\exp(\theta^\top \phi(s,a'))}{\sum_{a'\in\mathcal{A}}\exp((\theta+\beta w)^\top \phi(s,a'))}\\
    &=\pi(a|s)\frac{\exp(\beta w^\top \phi(s,a))}{\left[\frac{\sum_{a'\in\mathcal{A}}\exp(\theta^\top \phi(s,a'))\exp(w^\top \phi(s,a'))}{\sum_{a'\in\mathcal{A}}\exp(\theta^\top \phi(s,a'))} \right]}\\
    &=\pi(a|s)\frac{\exp(\beta w^\top \phi(s,a))}{\sum_{a'\in\mathcal{A}}\pi_t(a'|s)\exp(w^\top \phi(s,a'))}.
\end{align*}
This establish the equivalence between the two update equations.

\subsection{Proof of Proposition \ref{prop:actor_bound}}\label{pf:prop:actor_bound}
Using Lemma \ref{le:rewrite}, we see that the update equation of the actor (line 4 of Algorithm \ref{algo}) can be equivalently written by
\begin{align}
    \pi_{t+1}(a|s)&=\pi_t(a|s)\frac{\exp(\beta w_t^\top \phi(s,a))}{\sum_{a'\in\mathcal{A}}\pi_t(a'|s)\exp(\beta w_t^\top \phi(s,a'))}\nonumber\\
    &=\pi_t(a|s)\frac{\exp(\beta (w_t^\top \phi(s,a)-V^{\pi_t}(s)))}{\sum_{a'\in\mathcal{A}}\pi_t(a'|s)\exp(\beta (w_t^\top \phi(s,a')-V^{\pi_t}(s)))}\nonumber\\
    &=\pi_t(a|s)\frac{\exp(\beta( w_t^\top \phi(s,a)-V^{\pi_t}(s)))}{Z_t(s)},\label{eq:2}
\end{align}
where $Z_t(s)=\sum_{a'\in\mathcal{A}}\pi_t(a'|s)\exp(\beta (w_t^\top \phi(s,a')-V^{\pi_t}(s)))$. We will use Eq. (\ref{eq:2}) for our analysis. To prove Proposition \ref{prop:actor_bound}, we need the following sequence of lemmas.

\begin{lemma}\label{le:logzt}

For any $t\geq 0$ and $s\in\mathcal{S}$, we have the following lower bound for $\log (Z_t(s))$
\begin{align*}
	\log(Z_{t}(s))\geq\beta\sum_{a\in\mathcal{A}}\pi_t(a|s)(w_t^\top \phi(s,a)-Q^{\pi_t}(s,a)).
\end{align*}
\end{lemma}
\begin{proof}[Proof of Lemma \ref{le:logzt}]
Using the equivalent update rule (\ref{eq:2}) of $\pi_{t}$ and we have for any $t\geq 0$ and $s\in\mathcal{S}$:
\begin{align*}
	\log(Z_{t}(s))&=\log\left[\sum_{a\in\mathcal{A}}\pi_t(a|s)\exp(\beta (w_t^\top\phi(s,a)-V^{\pi_t}(s)))\right]\\
	&\geq \beta\sum_{a\in\mathcal{A}}\pi_t(a|s)(w_t^\top\phi(s,a)-V^{\pi_t}(s)))\tag{Jensen's inequality}\\
	&= \beta\sum_{a\in\mathcal{A}}\pi_t(a|s)(w_t^\top\phi(s,a)-Q^{\pi_t}(s,a)+Q^{\pi_t}(s,a)-V^{\pi_t}(s)))\\
	&=\beta\sum_{a\in\mathcal{A}}\pi_t(a|s)(w_t^\top\phi(s,a)-Q^{\pi_t}(s,a)),
\end{align*}
where in the last line we used $\sum_{a\in\mathcal{A}}\pi_t(a|s)Q^{\pi_t}(s,a)=V^{\pi_t}(s)$.
\end{proof}

For any starting distribution $\mu$ and policy $\pi$, we define the following as the discounted visitation distribution. 
\[
d^\pi_\mu(s) = (1-\gamma)\mathbb{E}_{s_0\sim\mu}\left[\sum_{t=0}^\infty \gamma^tPr^\pi(S_t=s|S_0 = s_0)\right].
\]

\begin{lemma}\label{le:Vmu}
For any starting distribution $\mu$, the following inequality holds:
\begin{align*}
	V^{\pi_{t+1}}(\mu)-V^{\pi_t}(\mu)\geq\;&\frac{1}{1-\gamma}\mathbb{E}_{s\sim d^{t+1}}\sum_{a\in\mathcal{A}}(\pi_t(a|s)-\pi_{t+1}(a|s))(w_t^\top\phi(s,a)-Q^{\pi_t}(s,a))\\
	&-\mathbb{E}_{s\sim \mu}\sum_{a\in\mathcal{A}}\pi_t(a|s)(w_t^\top\phi(s,a)-Q^{\pi_t}(s,a))+\frac{1}{\beta}\mathbb{E}_{s\sim \mu}\log Z_{t}(s),
\end{align*}
where for the ease of notation we denote $d^{\pi_t}_\mu\equiv d^t$.
\end{lemma}
\begin{proof}[Proof of Lemma \ref{le:Vmu}]
For any starting distribution $\mu$, we have
\begin{align}
	&V^{\pi_{t+1}}(\mu)-V^{\pi_t}(\mu)\nonumber\\
	=\;&\frac{1}{1-\gamma}\mathbb{E}_{s\sim d^{t+1}}\sum_{a\in\mathcal{A}}\pi_{t+1}(a|s)A^{\pi_t}(s,a)\tag{Performance Difference Lemma, \citep[Lemma 3.2]{agarwal2021theory}}\nonumber\\
	=\;&\frac{1}{1-\gamma}\mathbb{E}_{s\sim d^{t+1}}\sum_{a\in\mathcal{A}}\pi_{t+1}(a|s)(Q^{\pi_t}(s,a)-w_t^\top\phi(s,a)+w_t^\top\phi(s,a)-V^{\pi_t}(s))\nonumber\\
	=\;&\frac{1}{1-\gamma}\mathbb{E}_{s\sim d^{t+1}}\sum_{a\in\mathcal{A}}\pi_{t+1}(a|s)(Q^{\pi_t}(s,a)-w_t^\top\phi(s,a))\nonumber\\
	&+\frac{1}{(1-\gamma)\beta}\mathbb{E}_{s\sim d^{t+1}}\sum_{a\in\mathcal{A}}\pi_{t+1}(a|s)\log\left(\frac{\pi_{t+1}(a|s)}{\pi_t(a|s)}Z_{t}(s)\right).\label{eq:3}
\end{align}
Consider the second term on the RHS of the previous inequality. Using the definition of Kullback–Leibler (KL) divergence, we have
\begin{align*}
    &\mathbb{E}_{s\sim d^{t+1}}\sum_{a\in\mathcal{A}}\pi_{t+1}(a|s)\log\left(\frac{\pi_{t+1}(a|s)}{\pi_t(a|s)}Z_{t}(s)\right)\\
	=\;&\mathbb{E}_{s\sim d^{t+1}}D_{\text{KL}}(\pi_{t+1}(\cdot|s)\mid \pi_t(\cdot|s))+\mathbb{E}_{s\sim d^{t+1}}\log Z_{t}(s)\\
	\geq \;&\mathbb{E}_{s\sim d^{t+1}}\left[\log Z_{t}(s)-\beta\sum_{a\in\mathcal{A}}\pi_t(a|s)(w_t^\top\phi(s,a)-Q^{\pi_t}(s,a))\right]\\
	&+\beta\mathbb{E}_{s\sim d^{t+1}}\sum_{a\in\mathcal{A}}\pi_t(a|s)(w_t^\top\phi(s,a)-Q^{\pi_t}(s,a))\tag{KL divergence is non-negative}\\
	\geq \;&(1-\gamma)\mathbb{E}_{s\sim \mu}\left[\log Z_{t}(s)-\beta\sum_{a\in\mathcal{A}}\pi_t(a|s)(w_t^\top\phi(s,a)-Q^{\pi_t}(s,a))\right]\\
	&+\beta\mathbb{E}_{s\sim d^{t+1}}\sum_{a\in\mathcal{A}}\pi_t(a|s)(w_t^\top\phi(s,a)-Q^{\pi_t}(s,a))\tag{$d^{t+1}\geq (1-\gamma)\mu$ and Lemma \ref{le:logzt}}.
\end{align*}
By substituting the previous inequality into Eq. \eqref{eq:3} we obtain
\begin{align*}
    V^{\pi_{t+1}}(\mu)-V^{\pi_t}(\mu)\geq  \;&\frac{1}{1-\gamma}\mathbb{E}_{s\sim d^{t+1}}\sum_{a\in\mathcal{A}}(\pi_t(a|s)-\pi_{t+1}(a|s))(w_t^\top\phi(s,a)-Q^{\pi_t}(s,a))\\
	&-\mathbb{E}_{s\sim \mu}\sum_{a\in\mathcal{A}}\pi_t(a|s)(w_t^\top\phi(s,a)-Q^{\pi_t}(s,a))+\frac{1}{\beta}\mathbb{E}_{s\sim \mu}\log Z_{t}(s).
\end{align*}

\end{proof}

\begin{lemma}\label{le:Vnu}
The following equality holds for any starting distribution $\mu$ and $t\geq 0$:
\begin{align*}
	&V^{\pi^*}(\mu)-V^{\pi_t}(\mu)\\
	=\;&\frac{1}{1-\gamma}\mathbb{E}_{s\sim d^*}\sum_{a\in\mathcal{A}}\pi^*(a|s)(Q^{\pi_t}(s,a)-w_t^\top\phi(s,a))+\frac{1}{(1-\gamma)\beta}\mathbb{E}_{s\sim d^*}\log(Z_{t}(s))\\
	&+\frac{1}{(1-\gamma)\beta}\mathbb{E}_{s\sim d^*}\left[D_{\text{KL}}(\pi^*(\cdot|s)\mid \pi_t(\cdot|s))-D_{\text{KL}}(\pi^*(\cdot|s)\mid \pi_{t+1}(\cdot|s))\right],
\end{align*}
where $d^*\equiv d^{\pi^*}_\mu$ is the discounted visitation distribution corresponding to the optimal policy.
\end{lemma}

\begin{proof}[Proof of Lemma \ref{le:Vnu}]
Using the equivalent update rule of $\pi_{t}$ in \eqref{eq:2},  for any $t\geq 0$ and $s\in\mathcal{S}$ we have
\begin{align}
	&V^{\pi^*}(\mu)-V^{\pi_t}(\mu)\nonumber\\
	=\;&\frac{1}{1-\gamma}\mathbb{E}_{s\sim d^*}\sum_{a\in\mathcal{A}}\pi^*(a|s)A^{\pi_t}(s,a)\tag{Performance Difference Lemma}\nonumber\\
	=\;&\frac{1}{1-\gamma}\mathbb{E}_{s\sim d^*}\sum_{a\in\mathcal{A}}\pi^*(a|s)(Q^{\pi_t}(s,a)-w_t^\top \phi(s,a)+w_t^\top \phi(s,a)-V^{\pi_t}(s))\nonumber\\
	=\;&\frac{1}{1-\gamma}\mathbb{E}_{s\sim d^*}\sum_{a\in\mathcal{A}}\pi^*(a|s)(Q^{\pi_t}(s,a)-w_t^\top \phi(s,a))\nonumber\\
	&+\frac{1}{(1-\gamma)\beta}\mathbb{E}_{s\sim d^*}\sum_{a\in\mathcal{A}}\pi^*(a|s)\log\left(\frac{\pi_{t+1}(a|s)}{\pi_t(a|s)}Z_{t}(s)\right)\nonumber\\
	=\;&\frac{1}{1-\gamma}\mathbb{E}_{s\sim d^*}\sum_{a\in\mathcal{A}}\pi^*(a|s)(Q^{\pi_t}(s,a)-w_t^\top \phi(s,a))+\frac{1}{(1-\gamma)\beta}\mathbb{E}_{s\sim d^*}\log(Z_{t}(s))\nonumber\\
	&+\frac{1}{(1-\gamma)\beta}\mathbb{E}_{s\sim d^*}\left[D_{\text{KL}}(\pi^*(\cdot|s)\mid \pi_t(\cdot|s))-D_{\text{KL}}(\pi^*(\cdot|s)\mid \pi_{t+1}(\cdot|s))\right],\label{eq:Delta_V}
\end{align}
where the last line follows from the definition of KL divergence.
\end{proof}

We now proceed to prove Proposition \ref{prop:actor_bound}. Since Lemma \ref{le:Vmu} holds for any distribution $\mu$, apply lemma \ref{le:Vmu} with $\mu=d^*$ and we have 
\begin{align*}
    V^{\pi_{t+1}}(d^*)-V^{\pi_t}(d^*)\geq\;&\frac{1}{1-\gamma}\mathbb{E}_{s\sim d^{t+1}_{d^*}}\sum_{a\in\mathcal{A}}(\pi_t(a|s)-\pi_{t+1}(a|s))(w_t^\top\phi(s,a)-Q^{\pi_t}(s,a))\\
	&-\mathbb{E}_{s\sim d^*}\sum_{a\in\mathcal{A}}\pi_t(a|s)(w_t^\top\phi(s,a)-Q^{\pi_t}(s,a))+\frac{1}{\beta}\mathbb{E}_{s\sim d^*}\log Z_{t}(s),
\end{align*}
which implies
\begin{align}\label{eq:10}
    \frac{1}{\beta}\mathbb{E}_{s\sim d^*}\log Z_{t}(s)\leq V^{\pi_{t+1}}(d^*)-V^{\pi_t}(d^*)+\frac{3}{1-\gamma}\|\Phi w_t-Q^{\pi_t}\|_\infty.
\end{align}
Using \eqref{eq:Delta_V}, for any $T\geq 1$ we have
\begin{align*}
	&\sum_{t=0}^{T-1}(V^{\pi^*}(\mu)-V^{\pi_t}(\mu))\\
	= \;&\frac{1}{1-\gamma}\sum_{t=0}^{T-1}\mathbb{E}_{s\sim d^*}\sum_{a\in\mathcal{A}}\pi^*(a|s)(Q^{\pi_t}(s,a)-w_t^\top\phi(s,a))+\frac{1}{(1-\gamma)\beta}\sum_{t=0}^{T-1}\mathbb{E}_{s\sim d^*}\log(Z_{t}(s))\\
	&+\frac{1}{(1-\gamma)\beta}\sum_{t=0}^{T-1}\mathbb{E}_{s\sim d^*}\left[D_{\text{KL}}(\pi^*(\cdot|s)\mid \pi_t(\cdot|s))-D_{\text{KL}}(\pi^*(\cdot|s)\mid \pi_{t+1}(\cdot|s))\right]\\
	\leq \;&\frac{1}{1-\gamma}\sum_{t=0}^{T-1}\|Q^{\pi_t}-\Phi w_t\|_\infty+\frac{1}{1-\gamma}\sum_{t=0}^{T-1}\left[V^{\pi_{t+1}}(d^*)-V^{\pi_t}(d^*) + \frac{3}{1-\gamma}\|Q^{\pi_t}-\Phi w_t\|_\infty \right]\tag{Eq. (\ref{eq:10})}\\
	&+\frac{1}{(1-\gamma)\beta}\sum_{t=0}^{T-1}\mathbb{E}_{s\sim d^*}\left[D_{\text{KL}}(\pi^*(\cdot|s)\mid \pi_t(\cdot|s))-D_{\text{KL}}(\pi^*(\cdot|s)\mid \pi_{t+1}(\cdot|s))\right]\\
	\leq \;&\frac{1}{1-\gamma}\sum_{t=0}^{T-1}\|Q^{\pi_t}-\Phi w_t\|_\infty+\frac{1}{1-\gamma}(V^{\pi_T}(d^*)-V^{\pi_0}(d^*))+\frac{3}{(1-\gamma)^2}\sum_{t=0}^{T-1}\|Q^{\pi_t}-\Phi w_t\|_\infty\\
	&+\frac{1}{(1-\gamma)\beta}\mathbb{E}_{s\sim d^*}\left[D_{\text{KL}}(\pi^*(\cdot|s)\mid \pi_0(\cdot|s))-D_{\text{KL}}(\pi^*(\cdot|s)\mid \pi_{T}(\cdot|s))\right]\\
	\leq \;&\frac{4}{(1-\gamma)^2}\sum_{t=0}^{T-1}\|Q^{\pi_t}-\Phi w_t\|_\infty+\frac{1}{(1-\gamma)^2}+\frac{\log(\mathcal{A})}{(1-\gamma)\beta}\\
	\leq \;&\frac{4}{(1-\gamma)^2}\sum_{t=0}^{T-1}\|Q^{\pi_t}-\Phi w_t\|_\infty+\frac{2}{(1-\gamma)^2},
\end{align*}
where the last line follows from $\beta=\log(|\mathcal{A}|)$. 
Therefore, using the previous inequality and the definition of $\hat{T}$, we have:
\begin{align*}
	V^{\pi^*}(\mu)-\mathbb{E}\left[V^{\pi_{\hat{T}}}(\mu)\right]=\;& V^{\pi^*}(\mu)-\frac{1}{T}\sum_{t=0}^{T-1}\mathbb{E}\left[V^{\pi_t}(\mu)\right]\\
	\leq\;& \frac{2}{(1-\gamma)^2 T}+\frac{4}{(1-\gamma)^2T}\sum_{t=0}^{T-1}\mathbb{E}[\|Q^{\pi_t}-\Phi w_t\|_\infty],
\end{align*}
which proves Proposition \ref{prop:actor_bound}.

\subsection{Proof of Theorem \ref{thm:NAC}}\label{pf:thm:NAC}
Using the result of Proposition \ref{prop:actor_bound}, for any starting distribution $\mu$, we have:
\begin{align}
	V^{\pi^*}(\mu)-\mathbb{E}\left[V^{\pi_{\hat{T}}}(\mu)\right]
	\leq\;& \frac{2}{(1-\gamma)^2 T}+\frac{4}{(1-\gamma)^2T}\sum_{t=0}^{T-1}\mathbb{E}[\|Q^{\pi_t}-\Phi w_t\|_\infty]\nonumber\\
	\leq\;& \frac{2}{(1-\gamma)^2 T}+\frac{4}{(1-\gamma)^2T}\sum_{t=0}^{T-1}\mathbb{E}[\|Q^{\pi_t}-\Phi w_{\pi_t}\|_\infty]\nonumber\\
	&+\frac{4}{(1-\gamma)^2T}\sum_{t=0}^{T-1}\mathbb{E}[\|w_t-w_{\pi_t}\|_\infty]\nonumber\\
	\leq\;& \frac{2}{(1-\gamma)^2 T}+\frac{4\xi}{(1-\gamma)^2}\nonumber\\
	&+\frac{4}{(1-\gamma)^2T}\sum_{t=0}^{T-1}\mathbb{E}[\|w_t-w_{\pi_t}\|_\infty],\label{eq:4}
\end{align}
where we recall that $\xi=\max_{\theta}\|Q^{\pi_\theta}-\Phi w_{\pi_\theta}\|_\infty$. 

To control $\mathbb{E}[\|w_t-w_{\pi_t}\|_\infty]$, we apply Theorem \ref{thm:critic}. Since we choose the initial iterate $w_0=0$ in the critic, we can upper bound the constants $c_1$ and $c_2$ in Theorem \ref{thm:critic} by 
\begin{align*}
    c_1\leq c_3^2\quad \text{and}\quad c_2\leq 114 f(\gamma\zeta_{\max})^2c_3^2,
\end{align*}
where $c_3=1+\max_\pi\|w_\pi\|_2$. The following lemma provides a uniform bound on $\|w_\pi\|_2$ for any target policy $\pi$. The proof is presented in Appendix \ref{pf:le:wpi}.

\begin{lemma}\label{le:wpi}
For any policy $\pi$, we have $\|w_\pi\|_2\leq \frac{2}{(1-\gamma)\sqrt{1-\gamma_c}\sqrt{\lambda_{\min}}}$.
\end{lemma}

Since $\pi_t$ is determined by $\{(S_i,A_i)\}_{0\leq i\leq t(K+n)}$ while $w_t$ is determined by $\{(S_i,A_i)\}_{t(K+n)\leq i\leq (t+1)(K+n)}$, using the Markov property and conditional expectation, by Theorem 2.1 we have
\begin{align*}
    \mathbb{E}[\|w_t-w_{\pi_t}\|_\infty]\leq\;& \mathbb{E}[\|w_t-w_{\pi_t}\|_2]\\
    \leq\;& \sqrt{\mathbb{E}[\|w_t-w_{\pi_t}\|_2^2]}\tag{Jensen's inequality}\\
    \leq\;& c_3(1-(1-\gamma_c)\lambda_{\min} \alpha)^{\frac{K-(t_\alpha+n+1)}{2}}+\frac{11c_3 f(\gamma\zeta_{\max})[\alpha (t_\alpha+n+1)]^{1/2}}{(1-\gamma_c)^{1/2}\lambda_{\min}^{1/2}}.
\end{align*}
Finally, by substituting the previous inequality into Eq. \eqref{eq:4}, we get
\begin{align*}
    V^{\pi^*}(\mu)-\mathbb{E}\left[V^{\pi_{\hat{T}}}(\mu)\right]\leq\;& \frac{2}{(1-\gamma)^2 T}+\frac{4\xi}{(1-\gamma)^2}\\
	&+\frac{4c_3}{(1-\gamma)^2}(1-(1-\gamma_c)\lambda_{\min} \alpha)^{\frac{K-(t_\alpha+n+1)}{2}}\\
    &+\frac{44 c_3 f(\gamma\zeta_{\max})[\alpha (t_\alpha+n+1)]^{1/2}}{(1-\gamma_c)^{1/2}(1-\gamma)^2\lambda_{\min}^{1/2}}.
\end{align*}
This proves Theorem \ref{thm:NAC}.

\subsection{Proof of Lemma \ref{le:wpi}}\label{pf:le:wpi}

For any policy $\pi$, using the projected Bellman equation $\Phi w_\pi=\Pi_{\kappa_b}\mathcal{T}_\pi^n(\Phi w_\pi)$ we have
\begin{align*}
    \|Q^\pi-\Phi w_\pi\|_{\kappa_b}^2&=\|Q^\pi-\Pi_{\kappa_b}Q^\pi+\Pi_{\kappa_b}Q^\pi-\Phi w_\pi\|_{\kappa_b}^2\\
    &=\|Q^\pi-\Pi_{\kappa_b}Q^\pi\|_{\kappa_b}^2+\|\Phi w_\pi-\Pi_{\kappa_b}Q^\pi\|_{\kappa_b}^2\tag{Babylonian–Pythagorean theorem}\\
    &=\|Q^\pi-\Pi_{\kappa_b}Q^\pi\|_{\kappa_b}^2+\|\Pi_{\kappa_b}\mathcal{T}_\pi^n(\Phi w_\pi)-\Pi_{\kappa_b}\mathcal{T}_\pi^n(Q^\pi)\|_{\kappa_b}^2\\
    &\leq \|Q^\pi-\Pi_{\kappa_b}Q^\pi\|_{\kappa_b}^2+\gamma_c^2\|\Phi w_\pi-Q^\pi\|_{\kappa_b}^2.
\end{align*}

It follows that
\begin{align*}
    \|Q^\pi-\Phi w_\pi\|_{\kappa_b}&\leq \frac{1}{\sqrt{1-\gamma_c^2}}\|\Pi_{\kappa_b}Q^\pi-Q^\pi\|_{\kappa_b}\\
    &\leq \frac{1}{\sqrt{1-\gamma_c^2}}\|Q^\pi\|_{\kappa_b}\tag{Babylonian–Pythagorean theorem}\\
    &\leq \frac{1}{(1-\gamma)\sqrt{1-\gamma_c^2}}.
\end{align*}
Using the reverse triangle inequality we get
\begin{align*}
    \|w_\pi\|_2&\leq \frac{1}{\sqrt{\lambda_{\min}}}\|\Phi w_\pi\|_{\kappa_b}\\
    &\leq \frac{1}{\sqrt{\lambda_{\min}}}\left(\|Q^\pi\|_{\kappa_b}+\frac{1}{(1-\gamma)\sqrt{1-\gamma_c^2}}\right)\\
    &\leq \frac{2}{(1-\gamma)\sqrt{1-\gamma_c^2}\sqrt{\lambda_{\min}}}\\
    &\leq \frac{2}{(1-\gamma)\sqrt{1-\gamma_c}\sqrt{\lambda_{\min}}}.
\end{align*}

\subsection{Proof of Corollary \ref{co:sc}}\label{pf:co:sc}
For an given accuracy $\epsilon>0$, in order to achieve
\begin{align*}
    V^{\pi^*}(\mu)-\mathbb{E}\left[V^{\pi_{\hat{T}}}(\mu)\right]\leq \epsilon+\frac{3\xi}{(1-\gamma)^2},
\end{align*}
in light of Theorem \ref{thm:NAC} and Lemma \ref{le:wpi}, we must have
\begin{align*}
T&\sim   \mathcal{O}\left(\frac{1}{\epsilon(1-\gamma)^2}\right)\\
    \alpha &\sim \mathcal{O}\left(\frac{\epsilon^2}{\log(1/\epsilon)}\right)\Tilde{\mathcal{O}}\left(\frac{(1-\gamma_c)^2(1-\gamma)^6\lambda_{\min}^2}{nf(\gamma\zeta_{\max})^2}\right)\\
    K&\sim \mathcal{O}\left(\frac{\log^2(1/\epsilon)}{\epsilon^2}\right)\Tilde{\mathcal{O}}\left(\frac{nf(\gamma\zeta_{\max})^2}{(1-\gamma_c)^3(1-\gamma)^6\lambda_{\min}^3}\right).
\end{align*}
Therefore, the total sample complexity is 
\begin{align*}
    TK=\mathcal{O}\left(\frac{\log^2(1/\epsilon)}{\epsilon^3}\right)\Tilde{\mathcal{O}}\left(\frac{nf(\gamma\zeta_{\max})^2}{(1-\gamma_c)^3(1-\gamma)^8\lambda_{\min}^3}\right).
\end{align*}

\section{Discussion about Sample Complexity}\label{ap:sample_complexity}
For completeness, We restate here the  argument from \cite[Appendix C]{khodadadian2021finite2} that explains issues with definition of sample complexity when the error is not going to zero.  Consider a convergence bound of the form
\begin{align*}
    \text{Error}\leq \frac{1}{T}+\mathcal{E}_0,
\end{align*}
where $\mathcal{E}_0$ is a constant bias term, and $T$ is the number of iterations. For example, in our case, $\mathcal{E}_0$ represents the function approximation bias. By using the AM-GM inequality, we have
\begin{align}
    \text{Error}&\leq \left(\frac{1}{\mathcal{E}_0^{N-1}T^N}\mathcal{E}_0^{N-1}\right)^{1/N}+\mathcal{E}_0\nonumber\\
    &\leq \frac{1}{N\mathcal{E}_0^{N-1}}\frac{1}{T^N}+\left(2-\frac{1}{N}\right)\mathcal{E}_0\label{eq:error}\\
    &=\mathcal{O}\left(\frac{1}{T^N}\right)+\mathcal{O}(\mathcal{E}_0),\label{eq:error1}
\end{align}
which leads to the misleading interpretation of obtaining $\mathcal{O}(\epsilon^{-1/N})$ sample complexity for any $N\geq 1$. See Appendix C of \cite{khodadadian2021finite2} for a more detailed discussion.

A simple way to identify the problem in the previous derivation is to consider the special case where $\mathcal{E}_0=0$, which corresponds to using $\Phi=I_{|\mathcal{S}||\mathcal{A}|}$ in our NAC algorithm (i.e., the tabular setting). In this case, since the RHS of Eq. (\ref{eq:error}) is infinity, the convergence bound in Eq. (\ref{eq:error}) and also Eq. (\ref{eq:error1}) are meaningless. In our Theorem \ref{thm:NAC}, when $\Phi=I_{|\mathcal{S}||\mathcal{A}|}$ and hence $\xi=0$, Theorem \ref{thm:NAC} still provides a meaningful finite-sample bounds. In fact, it coincides with the finite-sample bounds of tabular NAC provided in \cite{khodadadian2021finite2} when the two truncation levels within the $Q$-trace algorithm are large enough. Therefore, the issue of trading off asymptotic error and convergence rate using AM-GM inequality is not present in our results. 

\section{Convergence of QNPG}
In this section we establish $\mathcal{O}(1/T)$ convergence of QNPG, improving upon the $\mathcal{O}(1/\sqrt{T})$ result in \cite{agarwal2021theory}. 

Consider an arbitrary (possibly dependent on policy $\pi$) distribution $\nu^\pi$ over the states of the MDP. For an arbitrary policy $\pi$, define 
\[
w^{\pi}\in\argmin_w \mathbb{E}_{s\sim\nu^{\pi},a\sim\pi(\cdot|s)}[(Q^{\pi}(s,a)-w^\top \phi(s,a))^2].
\]
Note that the solution to the projected Bellman equation \ref{eq:pbe} is denoted as $w_\pi$ which can in general be different from $w^\pi$.

The general QNPG algorithm is presented in Algorithm \ref{algo:QNPG}.

\begin{algorithm}[H]
\caption{General QNPG}\label{algo:QNPG}
	\begin{algorithmic}[1] 
		\STATE {\bfseries Input:} $T$, $\beta$, $\theta_0$, features $\phi_{s,a}\in\mathbb{R}^d$ for all $s,a$, distribution function $\pi\rightarrow\nu^\pi$ \\
		\FOR{$t=0,1,\dots,T-1$}
		\STATE Evaluate $w^{\pi_{\theta_t}}$
		\STATE $\theta_{t+1} =\theta_t + \beta w^{\pi_{\theta_t}}$
		\ENDFOR
		\STATE\textbf{Output:} $\theta_{\hat{T}}$, where $\hat{T}$ is uniformly sampled from $[0,T-1]$.
	\end{algorithmic}
\end{algorithm}
Define 
\[
\xi_{error}=\max_\pi \|Q^\pi-\Phi w^\pi\|_\infty.
\]

We have the following theorem:
\begin{theorem}\label{thm:QNPG}
The general QNPG Algorithm \ref{algo:QNPG} with step size $\beta\geq\log(|\mathcal{A}|)$ satisfies the following
\[
V^*-\mathbb{E}[V^{\pi_{\theta_{\hat{T}}}}] \leq \frac{2}{(1-\gamma)^2T}+\frac{4}{(1-\gamma)^2}\xi_{error},
\]
where the expectation is only with respect to the randomness in $\hat{T}$.
\end{theorem}
\subsection{Proof of Theorem \ref{thm:QNPG}}
Throughout this section, we denote $\pi_t\equiv\pi_{\theta_t}$. Using Lemma \ref{le:rewrite}, we have
\begin{align}
    \pi_{t+1}(a|s)&=\pi_t(a|s)\frac{\exp(\beta\phi(s,a)^\top w^{\pi_t} )}{\sum_{a'\in\mathcal{A}}\pi_t(a'|s)\exp(\beta \phi(s,a')^\top w^{\pi_t})}\nonumber\\
    &=\pi_t(a|s)\frac{\exp(\beta (\phi(s,a)^\top w^{\pi_t}-V^{\pi_t}(s)))}{\sum_{a'\in\mathcal{A}}\pi_t(a'|s)\exp(\beta (\phi(s,a')^\top w^{\pi_t}-V^{\pi_t}(s)))}\nonumber\\
    &=\pi_t(a|s)\frac{\exp(\beta( \phi(s,a)^\top w^{\pi_t}-V^{\pi_t}(s)))}{Z_t(s)},\label{eq:22}
\end{align}
where $Z_t(s)=\sum_{a'\in\mathcal{A}}\pi_t(a'|s)\exp(\beta (\phi(s,a')^\top w^{\pi_t}-V^{\pi_t}(s)))$. First, we state three supporting lemmas for the proof of Theorem \ref{thm:QNPG}.  

\begin{lemma}\label{le:logzt2}

For any $t\geq 0$ and $s\in\mathcal{S}$, we have the following lower bound for $\log (Z_t(s))$
\begin{align*}
	\log(Z_{t}(s))\geq\beta\sum_{a\in\mathcal{A}}\pi_t(a|s)(\phi(s,a)^\top w^{\pi_t}-Q^{\pi_t}(s,a)).
\end{align*}
\end{lemma}

\begin{lemma}\label{le:Vmu2}
For any starting distribution $\mu$ and $t\geq 0$, the following inequality holds:
\begin{align*}
	V^{\pi_{t+1}}(\mu)-V^{\pi_t}(\mu)\geq\;&\frac{1}{1-\gamma}\mathbb{E}_{s\sim d^{t+1}}\sum_{a\in\mathcal{A}}(\pi_t(a|s)-\pi_{t+1}(a|s))(\phi(s,a)^\top w^{\pi_t}-Q^{\pi_t}(s,a))\\
	&-\mathbb{E}_{s\sim \mu}\sum_{a\in\mathcal{A}}\pi_t(a|s)(\phi(s,a)^\top w^{\pi_t}-Q^{\pi_t}(s,a))+\frac{1}{\beta}\mathbb{E}_{s\sim \mu}\log Z_{t}(s),
\end{align*}
where for the ease of notation we denote $d^{\pi_t}_\mu\equiv d^t$.
\end{lemma}

\begin{lemma}\label{le:Vnu2}
The following equality holds for any starting distribution $\mu$ and $t\geq 0$:
\begin{align*}
	V^{\pi^*}(\mu)-V^{\pi_t}(\mu)
	=&\;\frac{1}{1-\gamma}\mathbb{E}_{s\sim d^*}\sum_{a\in\mathcal{A}}\pi^*(a|s)(Q^{\pi_t}(s,a)-\phi(s,a)^\top w^{\pi_t})+\frac{1}{(1-\gamma)\beta}\mathbb{E}_{s\sim d^*}\log(Z_{t}(s))\\
	&+\frac{1}{(1-\gamma)\beta}\mathbb{E}_{s\sim d^*}\left[D_{\text{KL}}(\pi^*(\cdot|s)\mid \pi_t(\cdot|s))-D_{\text{KL}}(\pi^*(\cdot|s)\mid \pi_{t+1}(\cdot|s))\right],
\end{align*}
where $d^*\equiv d^{\pi^*}_\mu$ is the discounted visitation distribution corresponding to the optimal policy.
\end{lemma}

We now proceed to prove Theorem \ref{thm:QNPG}. Since Lemma \ref{le:Vmu2} holds for any distribution $\mu$, apply this lemma with $\mu=d^*$ and we have 
\begin{align*}
    V^{\pi_{t+1}}(d^*)-V^{\pi_t}(d^*)\geq\;&\frac{1}{1-\gamma}\mathbb{E}_{s\sim d^{t+1}_{d^*}}\sum_{a\in\mathcal{A}}(\pi_t(a|s)-\pi_{t+1}(a|s))(\phi(s,a)^\top w^{\pi_t}-Q^{\pi_t}(s,a))\\
	&-\mathbb{E}_{s\sim d^*}\sum_{a\in\mathcal{A}}\pi_t(a|s)(\phi(s,a)^\top w^{\pi_t}-Q^{\pi_t}(s,a))+\frac{1}{\beta}\mathbb{E}_{s\sim d^*}\log Z_{t}(s),
\end{align*}
which implies
\begin{align}\label{eq:102}
    \frac{1}{\beta}\mathbb{E}_{s\sim d^*}\log Z_{t}(s)\leq V^{\pi_{t+1}}(d^*)-V^{\pi_t}(d^*)+\frac{3}{1-\gamma}\|\Phi w^{\pi_t}-Q^{\pi_t}\|_\infty.
\end{align}
Using Lemma \ref{le:Vnu2}, for any $T\geq 1$ we have
\begin{align*}
	&\sum_{t=0}^{T-1}(V^{\pi^*}(\mu)-V^{\pi_t}(\mu))\\
	= \;&\frac{1}{1-\gamma}\sum_{t=0}^{T-1}\mathbb{E}_{s\sim d^*}\sum_{a\in\mathcal{A}}\pi^*(a|s)(Q^{\pi_t}(s,a)-\phi(s,a)^\top w^{\pi_t})+\frac{1}{(1-\gamma)\beta}\sum_{t=0}^{T-1}\mathbb{E}_{s\sim d^*}\log(Z_{t}(s))\\
	&+\frac{1}{(1-\gamma)\beta}\sum_{t=0}^{T-1}\mathbb{E}_{s\sim d^*}\left[D_{\text{KL}}(\pi^*(\cdot|s)\mid \pi_t(\cdot|s))-D_{\text{KL}}(\pi^*(\cdot|s)\mid \pi_{t+1}(\cdot|s))\right]\\
	\leq \;&\frac{1}{1-\gamma}\sum_{t=0}^{T-1}\|Q^{\pi_t}-\Phi w^{\pi_t}\|_\infty+\frac{1}{1-\gamma}\sum_{t=0}^{T-1}\left[V^{\pi_{t+1}}(d^*)-V^{\pi_t}(d^*) + \frac{3}{1-\gamma}\|Q^{\pi_t}-\Phi w^{\pi_t}\|_\infty \right]\tag{Eq. (\ref{eq:102})}\\
	&+\frac{1}{(1-\gamma)\beta}\sum_{t=0}^{T-1}\mathbb{E}_{s\sim d^*}\left[D_{\text{KL}}(\pi^*(\cdot|s)\mid \pi_t(\cdot|s))-D_{\text{KL}}(\pi^*(\cdot|s)\mid \pi_{t+1}(\cdot|s))\right]\\
	\leq \;&\frac{1}{1-\gamma}\sum_{t=0}^{T-1}\|Q^{\pi_t}-\Phi w^{\pi_t}\|_\infty+\frac{1}{1-\gamma}(V^{\pi_T}(d^*)-V^{\pi_0}(d^*))+\frac{3}{(1-\gamma)^2}\sum_{t=0}^{T-1}\|Q^{\pi_t}-\Phi w^{\pi_t}\|_\infty\\
	&+\frac{1}{(1-\gamma)\beta}\mathbb{E}_{s\sim d^*}\left[D_{\text{KL}}(\pi^*(\cdot|s)\mid \pi_0(\cdot|s))-D_{\text{KL}}(\pi^*(\cdot|s)\mid \pi_{T}(\cdot|s))\right]\\
	\leq \;&\frac{4}{(1-\gamma)^2}\sum_{t=0}^{T-1}\|Q^{\pi_t}-\Phi w^{\pi_t}\|_\infty+\frac{1}{(1-\gamma)^2}+\frac{\log(\mathcal{A})}{(1-\gamma)\beta}\\
	\leq \;&\frac{4}{(1-\gamma)^2}\sum_{t=0}^{T-1}\|Q^{\pi_t}-\Phi w^{\pi_t}\|_\infty+\frac{2}{(1-\gamma)^2},
\end{align*}
where the last line follows from $\beta=\log(|\mathcal{A}|)$. 
Therefore, using the previous inequality and the definition of $\hat{T}$, we have:
\begin{align*}
	V^{\pi^*}(\mu)-\mathbb{E}\left[V^{\pi_{\hat{T}}}(\mu)\right]=\;& V^{\pi^*}(\mu)-\frac{1}{T}\sum_{t=0}^{T-1}\mathbb{E}\left[V^{\pi_t}(\mu)\right]\\
	\leq\;& \frac{2}{(1-\gamma)^2 T}+\frac{4}{(1-\gamma)^2T}\sum_{t=0}^{T-1}\mathbb{E}[\|Q^{\pi_t}-\Phi w^{\pi_t}\|_\infty]\\
	\leq\;& \frac{2}{(1-\gamma)^2 T}+\frac{4}{(1-\gamma)^2}\xi_{error}
\end{align*}

Since $w^{\pi}$ is defined as  $w^{\pi}\in\argmin_w \mathbb{E}_{s\sim\nu^{\pi},a\sim\pi(\cdot|s)}[(Q^{\pi}(s,a)-\phi(s,a)^\top w)^2]$, one might be interested in an upper bound based on the error 
\[
\epsilon_{bias}=\max_\pi \mathbb{E}_{s\sim\nu^{\pi},a\sim\pi(\cdot|s)}[(Q^{\pi}(s,a)-\phi(s,a)^\top w^{\pi})^2].
\]
The following Corollary characterizes this error.
\begin{corollary}\label{cor:QNPG}
The general QNPG Algorithm \ref{algo:QNPG} satisfies the following
\[
V^*-\mathbb{E}[V^{\pi_{\theta_{\hat{T}}}}] \leq \frac{2}{(1-\gamma)^2T}+\frac{4}{(1-\gamma)^2}\sqrt{\frac{\epsilon_{bias}}{\lambda}},
\]
where $\lambda=\min_{\pi,s,a}\nu^\pi(s)\pi(a|s)$ and the expectation is only with respect to the randomness in $\hat{T}$.
\end{corollary}
\begin{proof}[Proof of Corollary \ref{cor:QNPG}]
The proof follows immediately from Theorem \ref{thm:QNPG} and the norm inequality $\|Q^{\pi}-\Phi w^{\pi}\|_\infty\leq \frac{1}{\sqrt{\lambda}} \sqrt{\mathbb{E}_{s\sim\nu^{\pi},a\sim\pi(\cdot|s)}[(Q^{\pi}(s,a)-\phi(s,a)^\top w^{\pi})^2]}$
\end{proof}

\subsection{Proof of Auxiliary lemmas}
\begin{proof}[Proof of Lemma \ref{le:logzt2}]
Using the equivalent update rule (\ref{eq:2}) of $\pi_{t}$ and we have for any $t\geq 0$ and $s\in\mathcal{S}$:
\begin{align*}
	\log(Z_{t}(s))&=\log\left[\sum_{a\in\mathcal{A}}\pi_t(a|s)\exp(\beta (\phi(s,a)^\top w^{\pi_t}-V^{\pi_t}(s)))\right]\\
	&\geq \beta\sum_{a\in\mathcal{A}}\pi_t(a|s)(\phi(s,a)^\top w^{\pi_t}-V^{\pi_t}(s)))\tag{Jensen's inequality}\\
	&=\beta\sum_{a\in\mathcal{A}}\pi_t(a|s)(\phi(s,a)^\top w^{\pi_t}-Q^{\pi_t}(s,a)),
\end{align*}
where in the last line we used $\sum_{a\in\mathcal{A}}\pi_t(a|s)Q^{\pi_t}(s,a)=V^{\pi_t}(s)$.
\end{proof}

\begin{proof}[Proof of Lemma \ref{le:Vmu2}]
For any starting distribution $\mu$, we have
\begin{align}
	&V^{\pi_{t+1}}(\mu)-V^{\pi_t}(\mu)\nonumber\\
	=\;&\frac{1}{1-\gamma}\mathbb{E}_{s\sim d^{t+1}}\sum_{a\in\mathcal{A}}\pi_{t+1}(a|s)A^{\pi_t}(s,a)\tag{Performance Difference Lemma, \citep[Lemma 3.2]{agarwal2021theory}}\nonumber\\
	=\;&\frac{1}{1-\gamma}\mathbb{E}_{s\sim d^{t+1}}\sum_{a\in\mathcal{A}}\pi_{t+1}(a|s)(Q^{\pi_t}(s,a)-\phi(s,a)^\top w^{\pi_t}+\phi(s,a)^\top w^{\pi_t}-V^{\pi_t}(s))\nonumber\\
	=\;&\frac{1}{1-\gamma}\mathbb{E}_{s\sim d^{t+1}}\sum_{a\in\mathcal{A}}\pi_{t+1}(a|s)(Q^{\pi_t}(s,a)-\phi(s,a)^\top w^{\pi_t})\nonumber\\
	&+\frac{1}{(1-\gamma)\beta}\mathbb{E}_{s\sim d^{t+1}}\sum_{a\in\mathcal{A}}\pi_{t+1}(a|s)\log\left(\frac{\pi_{t+1}(a|s)}{\pi_t(a|s)}Z_{t}(s)\right).\label{eq:32}
\end{align}
Consider the second term on the RHS of the previous inequality. Using the definition of Kullback–Leibler (KL) divergence, we have
\begin{align*}
    &\mathbb{E}_{s\sim d^{t+1}}\sum_{a\in\mathcal{A}}\pi_{t+1}(a|s)\log\left(\frac{\pi_{t+1}(a|s)}{\pi_t(a|s)}Z_{t}(s)\right)\\
	=\;&\mathbb{E}_{s\sim d^{t+1}}D_{\text{KL}}(\pi_{t+1}(\cdot|s)\mid \pi_t(\cdot|s))+\mathbb{E}_{s\sim d^{t+1}}\log Z_{t}(s)\\
	\geq \;&\mathbb{E}_{s\sim d^{t+1}}\left[\log Z_{t}(s)-\beta\sum_{a\in\mathcal{A}}\pi_t(a|s)(\phi(s,a)^\top w^{\pi_t}-Q^{\pi_t}(s,a))\right]\\
	&+\beta\mathbb{E}_{s\sim d^{t+1}}\sum_{a\in\mathcal{A}}\pi_t(a|s)(\phi(s,a)^\top w^{\pi_t}-Q^{\pi_t}(s,a))\tag{KL divergence is non-negative}\\
	\geq \;&(1-\gamma)\mathbb{E}_{s\sim \mu}\left[\log Z_{t}(s)-\beta\sum_{a\in\mathcal{A}}\pi_t(a|s)(\phi(s,a)^\top w^{\pi_t}-Q^{\pi_t}(s,a))\right]\\
	&+\beta\mathbb{E}_{s\sim d^{t+1}}\sum_{a\in\mathcal{A}}\pi_t(a|s)(\phi(s,a)^\top w^{\pi_t}-Q^{\pi_t}(s,a))\tag{$d^{t+1}\geq (1-\gamma)\mu$ and Lemma \ref{le:logzt2}}.
\end{align*}
By substituting the previous inequality into Eq. \eqref{eq:32} we obtain
\begin{align*}
    V^{\pi_{t+1}}(\mu)-V^{\pi_t}(\mu)\geq  \;&\frac{1}{1-\gamma}\mathbb{E}_{s\sim d^{t+1}}\sum_{a\in\mathcal{A}}(\pi_t(a|s)-\pi_{t+1}(a|s))(\phi(s,a)^\top w^{\pi_t}-Q^{\pi_t}(s,a))\\
	&-\mathbb{E}_{s\sim \mu}\sum_{a\in\mathcal{A}}\pi_t(a|s)(\phi(s,a)^\top w^{\pi_t}-Q^{\pi_t}(s,a))+\frac{1}{\beta}\mathbb{E}_{s\sim \mu}\log Z_{t}(s).
\end{align*}
\end{proof}

\begin{proof}[Proof of Lemma \ref{le:Vnu2}]
Using the equivalent update rule of $\pi_{t}$ in Eq. \eqref{eq:22},  for any $t\geq 0$ and $s\in\mathcal{S}$ we have
\begin{align}
	&V^{\pi^*}(\mu)-V^{\pi_t}(\mu)\nonumber\\
	=\;&\frac{1}{1-\gamma}\mathbb{E}_{s\sim d^*}\sum_{a\in\mathcal{A}}\pi^*(a|s)A^{\pi_t}(s,a)\tag{Performance Difference Lemma}\nonumber\\
	=\;&\frac{1}{1-\gamma}\mathbb{E}_{s\sim d^*}\sum_{a\in\mathcal{A}}\pi^*(a|s)(Q^{\pi_t}(s,a)-\phi(s,a)^\top w^{\pi_t}+\phi(s,a)^\top w^{\pi_t}-V^{\pi_t}(s))\nonumber\\
	=\;&\frac{1}{1-\gamma}\mathbb{E}_{s\sim d^*}\sum_{a\in\mathcal{A}}\pi^*(a|s)(Q^{\pi_t}(s,a)-\phi(s,a)^\top w^{\pi_t})\nonumber\\
	&+\frac{1}{(1-\gamma)\beta}\mathbb{E}_{s\sim d^*}\sum_{a\in\mathcal{A}}\pi^*(a|s)\log\left(\frac{\pi_{t+1}(a|s)}{\pi_t(a|s)}Z_{t}(s)\right)\nonumber\\
	=\;&\frac{1}{1-\gamma}\mathbb{E}_{s\sim d^*}\sum_{a\in\mathcal{A}}\pi^*(a|s)(Q^{\pi_t}(s,a)-\phi(s,a)^\top w^{\pi_t})+\frac{1}{(1-\gamma)\beta}\mathbb{E}_{s\sim d^*}\log(Z_{t}(s))\nonumber\\
	&+\frac{1}{(1-\gamma)\beta}\mathbb{E}_{s\sim d^*}\left[D_{\text{KL}}(\pi^*(\cdot|s)\mid \pi_t(\cdot|s))-D_{\text{KL}}(\pi^*(\cdot|s)\mid \pi_{t+1}(\cdot|s))\right].\label{eq:Delta_V2}
\end{align}
\end{proof}
\section{Global Convergence With Linear \texorpdfstring{$Q$}--function}\label{ap:global_convergence}
Throughout this section we denote $\pi_t\equiv\pi_{\theta_t}$. Consider NPG Algorithm \ref{algo:NPG}.
\begin{algorithm}[H]
\caption{Natural Policy Gradient Algorithm}\label{algo:NPG}
	\begin{algorithmic}[1] 
		\STATE {\bfseries Input:} $T$, $\beta$, $\theta_0=0$, features $\phi(s,a)\in\mathbb{R}^d$ for all $s,a$ \\
		\FOR{$t=0,1,\dots,T-1$}
		\STATE Evaluate the unique solution of $\Phi w=\Pi_{\kappa_b}\mathcal{T}_{\pi_t}^n(\Phi w)$ as $w_{\pi_t}$
		\STATE $\theta_{t+1} =\theta_t + \beta w_{\pi_t}$
		\ENDFOR
		\STATE\textbf{Output:} $\theta_{T}$
	\end{algorithmic}
\end{algorithm}

In this section we prove the following fact.
\begin{fact}\label{fac:1}
Suppose the $Q$-function corresponding to all the policies in the parametrized space is linearly realizable. In other words, suppose $Q^{\pi_\theta}(s,a)= w_{\pi_\theta}^\top\phi(s,a)$ for all $s,a$ and $\theta\in\mathbb{R}^d$. Then, for an arbitrary distribution $\rho$, NPG Algorithm \ref{algo:NPG} converges to the global optimal policy as $V^{\pi^*}(\rho)-V^{\pi_{T}}(\rho)\leq \frac{\log(|\mathcal{A}|)}{(1-\gamma)\beta (T+1)} + \frac{1}{(1-\gamma)^2(T+1)}$.

\end{fact}
Two remarks regarding the Fact \ref{fac:1} are in order.  First of all, we should emphasize that this fact is evident from our convergence bound in Theorem \ref{thm:NAC}. In particular, due to the assumption on the feature vectors, it is easy to see that $\xi=0$. Furthermore, due to the deterministic update of Eq. \eqref{eq:NPG_update}, we can substitute $A_3=A_4=0$. Hence we have $1/T$ rate for global convergence of the update in Eq. \eqref{eq:NPG_update}. What we are doing in this section is to provide a different view point for this result. Furthermore, note that all the policies achieved through the NPG update lie within the space of parameterized policies. In particular, the parameter $\theta_t$ of the policy $\pi_t$ is equal to $\theta_t=\sum_{l=0}^{t-1}\beta w_{\pi_l}$.

\begin{proof}[Proof of Fact \ref{fac:1}]
By Lemma \ref{le:rewrite} it is easy to see that the update of Algorithm \ref{algo:NPG} is equivalent to the update of the policy as follows
\begin{align}\label{eq:NPG_update}
    \pi_{t+1}(a|s)=\frac{\pi_t(a|s)\exp(\beta w_{\pi_t}^\top\phi(s,a))}{\sum_{a'}\pi_t(a'|s)\exp(\beta w_{\pi_t}^\top\phi(s,a'))},
\end{align}
where $w_{\pi}$ is the solution of the projected Bellman equation \ref{eq:pbe}.

Denote $Z_t(s)=\sum_{a'}\pi_t(a'|s)\exp(\beta w_{\pi_t}^\top\phi(s,a'))$. We have
\begin{align*}
    \log Z_t(s) =& \log \sum_{a'}\pi_t(a'|s)\exp(\beta w_{\pi_t}^\top\phi(s,a'))\\
    \geq& \sum_{a'} \pi_t(a'|s) \beta w_{\pi_t}^\top\phi(s,a')\tag{Jensen's inequality}\\
    =&\sum_{a'} \pi_t(a'|s)\beta Q^{\pi_t}(s,a')\\
    =&\beta V^{\pi_t}(s).
\end{align*}
For any distribution $\mu$, denote $d^{t}=d^{\pi_t}_\mu$. We have
\begin{align*}
    V^{\pi_{t+1}}(\mu)-V^{\pi_{t}}(\mu) =& \frac{1}{1-\gamma}\mathbb{E}_{s\sim d^{t+1}}\sum_{a\in\mathcal{A}}\pi_{t+1}(a|s)A^{\pi_t}(s,a)\\
    =&\frac{1}{1-\gamma}\mathbb{E}_{s\sim d^{t+1}}\sum_{a\in\mathcal{A}}\pi_{t+1}(a|s)(Q^{\pi_t}(s,a)-V^{\pi_t}(s))\\
    =&\frac{1}{1-\gamma}\mathbb{E}_{s\sim d^{t+1}}\sum_{a\in\mathcal{A}}\pi_{t+1}(a|s)(w_{\pi_t}^\top\phi(s,a)-V^{\pi_t}(s))\\
    =&\frac{1}{1-\gamma}\mathbb{E}_{s\sim d^{t+1}}\sum_{a\in\mathcal{A}}\pi_{t+1}(a|s) w_{\pi_t}^\top\phi(s,a) - \frac{1}{1-\gamma}\mathbb{E}_{s\sim d^{t+1}}\sum_{a\in\mathcal{A}}\pi_{t+1}(a|s)V^{\pi_t}(s)\\
    =&\frac{1}{(1-\gamma)\beta}\mathbb{E}_{s\sim d^{t+1}}\sum_{a\in\mathcal{A}}\pi_{t+1}(a|s)\log\frac{\pi_{t+1}(a|s)Z_t(s)}{\pi_t(a|s)} - \frac{1}{1-\gamma}\mathbb{E}_{s\sim d^{t+1}}V^{\pi_t}(s)\\
    =&\frac{1}{(1-\gamma)\beta}\mathbb{E}_{s\sim d^{t+1}}\left[D_{KL}(\pi_{t+1}(\cdot|s)||\pi_t(\cdot|s))+\log Z_t(s)\right] - \frac{1}{1-\gamma}\mathbb{E}_{s\sim d^{t+1}}V^{\pi_t}(s)\\
    \geq & \frac{1}{(1-\gamma)\beta}\mathbb{E}_{s\sim d^{t+1}}\log Z_t(s) - \frac{1}{1-\gamma}\mathbb{E}_{s\sim d^{t+1}}V^{\pi_t}(s)\tag{positivity of KL-divergence}\\
    = & \frac{1}{1-\gamma}\mathbb{E}_{s\sim d^{t+1}}\left[\frac{1}{\beta}\log Z_t(s) - V^{\pi_t}(s)\right]\\
    \geq & \mathbb{E}_{s\sim \mu}\left[\frac{1}{\beta}\log Z_t(s) - V^{\pi_t}(s)\right]\geq0.\tag{by definition of $d^{t+1}$}
\end{align*}
Note that the above inequality shows monotonic improvement of the update in NPG.

For an arbitrary distribution $\rho$, denote $d^*_\rho=d^*$. We have 
\begin{align}
	V^{\pi^*}(\rho)-V^{\pi_t}(\rho)
	=\;&\frac{1}{1-\gamma}\mathbb{E}_{s\sim d^*}\sum_{a\in\mathcal{A}}\pi^*(a|s)A^{\pi_t}(s,a)\tag{Performance Difference Lemma}\nonumber\\
	=\;&\frac{1}{1-\gamma}\mathbb{E}_{s\sim d^*}\sum_{a\in\mathcal{A}}\pi^*(a|s)(Q^{\pi_t}(s,a)-V^{\pi_t}(s))\nonumber\\
	=\;&\frac{1}{1-\gamma}\mathbb{E}_{s\sim d^*}\sum_{a\in\mathcal{A}}\pi^*(a|s)(w_{\pi_t}^\top\phi(s,a)-V^{\pi_t}(s))\nonumber\\
	=&\frac{1}{(1-\gamma)\beta}\mathbb{E}_{s\sim d^*}\sum_{a\in\mathcal{A}}\pi^*(a|s)\log\left(\frac{\pi_{t+1}(a|s)}{\pi_t(a|s)}Z_{t}(s)\right)-\frac{1}{1-\gamma}\mathbb{E}_{s\sim d^*}V^{\pi_t}(s)\nonumber\\
	=&\frac{1}{(1-\gamma)\beta}\mathbb{E}_{s\sim d^*}\left[D_{\text{KL}}(\pi^*(\cdot|s)\mid \pi_t(\cdot|s))-D_{\text{KL}}(\pi^*(\cdot|s)\mid \pi_{t+1}(\cdot|s))\right]\nonumber\\
	&+\frac{1}{1-\gamma}\mathbb{E}_{s\sim d^*}\left[\frac{1}{\beta}\log\left(Z_{t}(s)\right)-V^{\pi_t}(s)\right]\nonumber\\
	\leq&\frac{1}{(1-\gamma)\beta}\mathbb{E}_{s\sim d^*}\left[D_{\text{KL}}(\pi^*(\cdot|s)\mid \pi_t(\cdot|s))-D_{\text{KL}}(\pi^*(\cdot|s)\mid \pi_{t+1}(\cdot|s))\right]\nonumber\\
	&+\frac{1}{1-\gamma}\mathbb{E}_{s\sim d^*}\left[V^{\pi_{t+1}}(d^*)-V^{\pi_{t}}(d^*)\right].\label{eq:NPG_V_gap}
\end{align}
Summing up both sides of the above inequality, we get 
\begin{align*}
    V^{\pi^*}(\rho)-V^{\pi_{T-1}}(\rho)\leq& \frac{1}{T}\sum_{t=0}^{T-1}V^{\pi^*}(\rho)-V^{\pi_{t}}(\rho)\tag{monotonic improvement of NPG}\\
    \leq& \frac{1}{(1-\gamma)\beta T}\mathbb{E}_{s\sim d^*}\left[D_{\text{KL}}(\pi^*(\cdot|s)\mid \pi_0(\cdot|s))-D_{\text{KL}}(\pi^*(\cdot|s)\mid \pi_{T}(\cdot|s))\right]\\
    &+\frac{1}{(1-\gamma)T}\mathbb{E}_{s\sim d^*}\left[V^{\pi_{T}}(d^*)-V^{\pi_{0}}(d^*)\right]\tag{by Eq. \eqref{eq:NPG_V_gap}}\\
    \leq & \frac{\log(|\mathcal{A}|)}{(1-\gamma)\beta T} + \frac{1}{(1-\gamma)^2T}.
\end{align*}

\end{proof}
\end{appendix}

\end{document}